%% file: main.tex
\documentclass{article} %
\usepackage{iclr2027_conference,times}

\usepackage{xcolor}
\usepackage{soul}
\sethlcolor{yellow}

\usepackage[ruled]{algorithm2e}
\usepackage{amsmath,amssymb}
\usepackage{float}
\usepackage{bm}

\usepackage{hyperref}
\usepackage{url}

\input{macros.tex}

\title{Fine-tuning on self-generated and reward-weighted data: Learning dynamics, Convergence rates, and benefits of off-policyness}

\author{Zhiwei Wang \thanks{Work done during an internship at Alibaba Group.} \\
Department of Mathematical Sciences \\
Tsinghua University \\
\texttt{zhiweithu@gmail.com} \\
\And
Yanxi Chen \\
Alibaba Group \\
\texttt{yxchen0@outlook.com}
\And
Yaliang Li \& Bolin Ding \\
Alibaba Group \\
\texttt{\{yaliang.li, bolin.ding\}@alibaba-inc.com} \\
}

\definecolor{todo}{RGB}{200,0,20}

\iclrfinalcopy %
\begin{document}

\maketitle

\begin{abstract}
\input{sections/abstract}

\end{abstract}

\input{sections/introduction}

\input{sections/main_results_high_level}

\input{sections/main_results_in_depth.tex}

\input{sections/related_work}

\input{sections/conclusion}

\bibliography{reference}
\bibliographystyle{iclr2027_conference}

\clearpage
\newpage

\appendix
\input{appendix/convergence_proofs}
\input{appendix/comparison_proofs}

\end{document}

%% file: macros.tex
\usepackage[T1]{fontenc}
\usepackage[utf8]{inputenc}
\usepackage{amsmath,amssymb,amsthm,mathtools}
\usepackage{graphicx,booktabs,microtype,xcolor}
\usepackage{hyperref}
\usepackage{xurl}
\hypersetup{colorlinks=true,citecolor=blue,linkcolor=blue,urlcolor=blue}
\newtheorem{theorem}{Theorem}[section]
\newtheorem{lemma}[theorem]{Lemma}

\theoremstyle{definition}

\theoremstyle{remark}
\newtheorem{remark}[theorem]{Remark}

\DeclareMathOperator{\diag}{diag}
\DeclareMathOperator{\logit}{logit}

\newcommand{\R}{\mathbb{R}}

\newcommand{\E}{\mathbb{E}}

\newcommand{\btheta}{{\theta}}

\newcommand{\mumin}{\mu_{\text{min}}}
\newcommand{\mumax}{\mu_{\text{max}}}
\newcommand{\lo}{\theta}

\newcommand{\rg}{D}
\newcommand{\ir}{d}

\newcommand{\btran}{\bar{b}}
\newcommand{\ttran}{\bar{t}}
\newcommand{\pb}{x}

\newcommand{\res}{{RE}($S$)\xspace}
\newcommand{\reone}{{RE}(1)\xspace}

\newcommand{\DKL}{\text{D}_{\text{KL}}}
\newcommand{\Acal}{\mathcal{A}}

\newcommand{\bmu}{{\mu}}
\newcommand{\bpi}{{\pi}}

\newcommand{\pitheta}{\pi_{\btheta}} %
\newcommand{\pithetaindex}[2]{\pi_{\btheta_{#1, #2}}} %
\newcommand{\pithetaaction}[1]{\pi_{\btheta}(#1)} %
\newcommand{\Preward}{P_{R}}
\newcommand{\nablatheta}{\nabla_{\btheta}}
\newcommand{\pitildemub}{\tilde{\pi}^{\mu}_{b}}

\newcommand{\Teps}{T_{\epsilon}}

%% file: sections/abstract.tex
We study the learning dynamics of fine-tuning a policy model on self-generated and reward-weighted rollout data,
with particular focus on a generalized version of classic REINFORCE --- referred to as \res --- that updates the rollout distribution once every $S \ge 1$ gradient steps.
Prior work in bandits and reinforcement learning has developed rich theory for REINFORCE and policy gradient methods, and on-policy sampling (i.e., a small value of $S$, ideally $1$) is often viewed as crucial to their success;
yet in prominent application like post-training large language models, reward-guided self-training has proved to be effective even when the rollout distribution is updated infrequently, but theoretical understanding remains limited for the learning dynamics and convergence properties of these off-policy methods.
To bridge these gaps, we develop a unified theory for \res that covers the full spectrum of $S \ge 1$: it can be interpreted as a stage-wise iterative optimization process, where each stage takes $S$ gradient steps for minimizing the Kullback–Leibler distance to a fixed reward-weighted rollout distribution.
For multi-arm bandits with softmax policies, our in-depth analysis and numerical experiments reveal three key findings:
(1)~for any fixed $S$, \res enjoys global convergence to the optimal policy as the number of rollout distribution updates $B = \lfloor T / S \rfloor \rightarrow \infty$, where $T$ denotes the total number of gradient steps;
(2)~we prove tight two-sided bounds for the convergence rate of \res, and show that its suboptimality gap achieves an asymptotic $\Theta(1 / T)$ convergence rate, while $S$ only affects the length of a burn-in phase;
(3)~when initialized at a weak policy with a small optimal-action probability, \reone gets trapped around suboptimal policies for a long period, whereas \res with a suitable $S$ avoids the detour and achieves significantly faster convergence to the global optimum, highlighting the benefits of off-policyness in this case.

%% file: sections/introduction.tex
\section{Introduction}
\label{sec:introduction}

Multi-armed bandits and reinforcement learning (RL) have a long history \citep{lattimore2020bandit,sutton1998reinforcement} and have achieved substantial empirical success, including the recent use of RL in post-training large language models (LLMs) \citep{ouyang2022training,touvron2023llama2openfoundation,openai-o1,deepseek-r1,kimiteam2025kimik15scalingreinforcement}.
Many algorithms in these areas can be viewed as fine-tuning a policy model $\pitheta$ on self-generated and reward-weighted data.
One prominent example is the classical REINFORCE algorithm \citep{williams1992simple} that, in the bandit setting, estimates the policy gradient by
$(1/N)\sum_{i=1}^{N} r_i\,\nabla_{\btheta}\log\pitheta(a_i)$, 
where $a_i \sim \pitheta$ is a sampled action and $r_i$ is the reward for taking action $a_i$.

This perspective highlights a gap between two major lines of methodology, practice, and theory:

(1) In the bandit and RL literature, there exists a rich theory for policy gradient and policy iteration methods, including convergence and sample-complexity guarantees \citep{sutton1999policy,agarwal2021theory,mei2020global,mei2023stochastic,lu2024towards}.
Much of this theory assumes on-policy sampling, where the rollout distribution matches the current policy under optimization.
Methods that account for (a limited degree of) off-policyness often require explicit algorithmic modifications, for example through trust-region regularization or clipping importance-sampling ratios \citep{schulman2015trust,schulman2017proximal,espeholt2018impala}.
In practice, keeping the rollout and training policies closely aligned is also an important design consideration for these methods, motivating substantial efforts to support frequent and efficient synchronization of LLM model weights in RL infrastructure \citep{verl,openrlhf,trinity}.

(2) However, in the broader reward-guided self-training paradigm, it is also common to update the rollout distribution \emph{infrequently} while fine-tuning the policy model on rollout samples weighted or filtered by their rewards, using standard log-likelihood objectives.
Successful examples in the LLM domain include STaR \citep{zelikman2022star}, RAFT \citep{dong2023raft}, ReST \citep{gulcehre2023reinforcedselftrainingrestlanguage}, ReST$^{\text{EM}}$ \citep{singh2024humandatascalingselftraining}, rejection-sampling fine-tuning in Llama~2 \citep{touvron2023llama2openfoundation}, AsymRE \citep{arnal2025asymmetric}, and more recent industrial practice.
Such methods can be implemented with simple infrastructure: an LLM inference engine \citep{kwon2023efficientmemorymanagementlarge,2024sglangefficientexecutionofstructuredlanguagemodelprograms} generates rollout data and periodically loads an updated policy checkpoint, while a standard supervised fine-tuning (SFT) engine performs gradient updates on the reward-weighted or reward-filtered data.
Despite their empirical success, theoretical understanding and guarantees for these off-policy methods remain limited.

These observations raise intriguing research questions. 
What are the connections and differences between these two lines of on-policy versus off-policy methodology?
How does rollout staleness affect policy learning?
What convergence guarantees can be established when the rollout distribution is synchronized with the trained policy infrequently?
In particular, can off-policyness possibly lead to faster convergence than on-policy learning?

\paragraph{Contributions and outline.}

Towards answering these questions, we focus on a simple algorithm termed \res and formalized in Algorithm~\ref{alg:res_general}, which is a generalized version of vanilla REINFORCE that updates the rollout distribution once every $S \ge 1$ gradient steps.
The parameter $S$ denotes the maximum staleness of a rollout sample with respect to the most updated policy model, and \reone reduces to standard on-policy REINFORCE.
We show in Section~\ref{sec:main_results_high_level} that \res can be interpreted as a stage-wise iterative optimization process, where each stage effectively takes $S$ gradient steps for minimizing the Kullback-Leibler (KL) distance to the reward-weighted rollout distribution.
Unlike on-policy REINFORCE --- which moves along the direction of standard policy gradient --- our interpretation shows that \res takes an alternative path towards the optimal policy.

In Section~\ref{sec:main_results_in_depth}, we further focus on $K$-arm bandits and policies with softmax parameterization, for in-depth analysis and formal theoretical guarantees.
Our key findings include the following:
\begin{enumerate}
\item Theorem~\ref{thm:global}: \res enjoys global convergence to the optimal policy as the total number of gradient steps $T \rightarrow \infty$ regardless of $S$, as long as the number of rollout distribution updates $B = \lfloor T / S \rfloor$ also grows unbounded;
\item Theorems~\ref{thm:rate-upper} and~\ref{thm:rate-lower}: We prove tight two-sided bounds showing that the suboptimality gap of \res achieves an asymptotic $\Theta(1 / T)$ convergence rate with respect to the number of gradient updates, while $S$ only affects the length of a burn-in phase;
\item Theorem~\ref{thm:benefits_off_policyness}: When initialized at a weak policy with a small optimal-action probability $x$, \reone provably gets trapped around suboptimal policies for $\Omega(x^{-(K-1)})$ steps, whereas \res with $S = \Theta(x^{-1})$ avoids such a detour and approaches the global optimum within $O(x^{-1} \log(1/x))$ gradient steps, highlighting the benefits of off-policyness in this case.
\end{enumerate}

Discussion on prior literature and its connections to our work can be found in Section~\ref{sec:related_work}.
Limitations of our work and directions for future research are discussed in Section~\ref{sec:conclusion}.
We hope that our theoretical study and numerical experiments can offer new insights for reward-guided self-training methodology, and inspire further research or practice in this area.

%% file: sections/main_results_high_level.tex
\section{A high-level understanding of RE(S) learning dynamics}
\label{sec:main_results_high_level}

We formalize the problem setting as follows.
Consider a multi-arm bandit with action space $\Acal$.
Let $\btheta$ denote the parameters of a policy, and $\pitheta$ denote its output distribution.
Every time an action $a \in \Acal$ is sampled from the policy distribution $a \sim \pitheta$, 
it receives a (stochastic) reward $r \sim \Preward(a)$ with expectation $\E[r] = \mu(a) \ge 0$.
We assume non-negative reward means $\mu$ throughout this work.
The mean reward of a policy $\pi$ is denoted by $J(\pi) \coloneqq \E_{a \sim \pi} \mu(a)$.

The \res algorithm is formalized in Algorithm~\ref{alg:res_general}.
It is a double-loop process: 
in the $b$-th outer iteration, a total of $N \times S$ action-reward pairs are sampled from the same rollout distribution $\pithetaindex{b}{0}$;
then, the inner loop takes $S$ gradient steps of reward-weighted fine-tuning, each consuming $N$ rollout samples.
It is obvious that the special case $S=1$ reduces to standard on-policy REINFORCE.

\begin{center}
\begin{minipage}{.9\linewidth}
\SetAlgoCaptionSeparator{}
\begin{algorithm}[H]
\caption{\res, a general version}
\label{alg:res_general}

\KwIn{initial policy parameters $\btheta_{0, 0}$, staleness $S \in \mathbb{N}_{+}$, learning rate $\eta > 0$.}

\For{$b = 0, 1, \dots, B-1$}{

    Sample $a_{s, i} \sim \pithetaindex{b}{0}$ and its reward $r_{s, i} \sim \Preward(a_{s, i})$, $0 \le s < S$, $1 \le i \le N$.
    
    \For{$s = 0, 1, \dots, S-1$}{

        $\btheta_{b, s+1} \gets \btheta_{b, s} + \frac{\eta}{N} \sum_{i=1}^{N} r_{s, i} \nablatheta \log \pitheta(a_{s, i}) |_{\btheta = \btheta_{b, s}}$.

    }

    Set $\btheta_{b+1, 0} \gets \btheta_{b, S}$
    
}

\end{algorithm}

\end{minipage}
\end{center}

While standard policy gradient theory says $\E_{a \sim \pitheta  } [ \mu(a) \nablatheta \log \pithetaaction{a} ] = \nablatheta J(\pitheta)$,
such unbiasedness is clearly not true for the gradient estimate in \res when off-policyness occurs, namely $S > 1$.
Fortunately, we can show that \res takes an alternative path towards the optimal policy.
Indeed, \res can be viewed as a stage-wise optimization process, where each stage takes $S$ gradient-descent steps for minimizing the KL distance to the reward-weighted rollout distribution $\pitildemub$:
\begin{align*}
    &\E_{a \sim \pithetaindex{b}{0}, r \sim \Preward(a)}[r \cdot \nablatheta \log \pithetaaction{a}] 
    =  \E_{a \sim \pithetaindex{b}{0}} [\mu(a) \cdot \nablatheta \log \pithetaaction{a}] \\
    &\qquad = Z_b \cdot \int_{a \in \Acal} \frac{\pithetaindex{b}{0}(a) \mu(a)}{Z_b} \nablatheta \log \pithetaaction{a} \, \text{d} a  
    = Z_b \cdot \E_{a \sim \pitildemub} \nablatheta \log \pithetaaction{a} \\
    &\qquad =  - Z_b \cdot \nablatheta \, \DKL\Big( \pitildemub \,\|\, \pitheta \Big),
\end{align*}
where $\pitildemub$ and $Z_b$ are defined by
\begin{align*}
    \pitildemub(a) \coloneqq \frac{\pithetaindex{b}{0}(a) \mu(a)}{Z_b}, \quad Z_b \coloneqq \int_{a \in \Acal} \pithetaindex{b}{0}(a) \mu(a) \, \text{d} a = \E_{a \sim \pithetaindex{b}{0} } \,\mu(a).
\end{align*}
If $S$ is sufficiently large, then by the end of the inner loop, $\pithetaindex{b+1}{0} = \pithetaindex{b}{S}$ will ideally get close to $\pitildemub$, which is a higher-reward policy than $\pithetaindex{b}{0}$.
Thus \res could converge to the optimal policy as the rollout distribution gets updated for $B \rightarrow \infty$ times.
In the next section, we formalize this intuition and provide in-depth convergence analysis in concrete settings.

%% file: sections/main_results_in_depth.tex
\section{In-depth analysis for bandits and softmax policies}
\label{sec:main_results_in_depth}

In this section, we present in-depth analysis for the learning dynamics and convergence rates of \res in bandits with softmax policies.
Our key findings include 
global convergence of \res to the optimal policy (Section~\ref{subsec:theory_global_convergence}),
two-sided bounds for convergence rates that reveal two distinct time scales in the learning dynamics of \res (Section~\ref{subsec:theory_convergence_rates}),
and the benefits of off-policyness in accelerating convergence under certain conditions (Section~\ref{subsec:theory_benefits_of_off_policyness}).
We conclude with numerical validation of these theoretical findings (Section~\ref{subsec:numerical_validation}).

\paragraph{Problem setup.}

Following the formulation in Section~\ref{sec:main_results_high_level},
we further assume a discrete action space $\Acal = [K]: = \{1,\ldots,K\}$, and a reward-mean vector $\bmu=[\mu(1),\ldots,\mu(K)]$.  
Assume that there is a unique optimal action, indexed by $1$ without loss of generality.
We let $\mumax/\mumin$ and $\Delta$ denote the maximum/minimum reward and the reward gap between the top-2 actions, respectively:
\begin{align*}
\mumax \coloneqq \mu(1) > \mu(a) \ge \mumin > 0, \quad 2 \le a \le K; \qquad
\Delta \coloneqq \mumax-\max_{2 \le a \le K}\mu(a)>0.
\end{align*}
We focus on policies with softmax parameterization throughout this section:
\begin{equation*}
\btheta = \big[\theta(1), \dots, \theta(K)\big] \in \R^{K}, \quad
\pithetaaction{a}=\frac{e^{\theta(a)}}{\sum_{j=1}^K e^{\theta(j)}}, \quad 1 \le a \le K.
\end{equation*}
Suppose that every coordinate of the initial logit vector $\btheta$ is finite, hence the initial policy has full support over $\Acal$.
The learner aims to find the policy $\pi_\lo$ that maximizes the objective function: 
\begin{equation}
J(\pi_\btheta) = \E_{a \sim \pitheta} \mu(a) = \langle \pitheta, \bmu \rangle.
\label{eq:objective function}
\end{equation}
For notational simplicity, we will use $\bpi_{b, s}$ and $\pithetaindex{b}{s}$ interchangeably.
In this concrete setting, the infinite-sample limit of the gradient estimate in Algorithm~\ref{alg:res_general} can be calculated by 
\begin{align*}
g_{b,s} \coloneqq \bpi_{b, 0} \odot \bmu - J(\bpi_{b, 0}) \bpi_{b, s},
\end{align*}
where $\odot$ denotes element-wise multiplication; see Lemma~\ref{lem:surrogate obj grad} in the appendix for the derivation.
This leads us to the specialized version of \res in Algorithm~\ref{alg:res_softmax}, which we investigate in this section.

\begin{center}
\begin{minipage}{0.8\linewidth}
\SetAlgoCaptionSeparator{}
\begin{algorithm}[H]
\caption{\res, softmax parameterization and infinite-sample limit}
\label{alg:res_softmax}

\KwIn{
    initial logits $\btheta_{0, 0} \in \R^{K}$, staleness $S \in \mathbb{N}_{+}$, learning rate $\eta > 0$.
}

\For{$b = 0, 1, \dots, B-1$}{

    \For{$s = 0, 1, \dots, S-1$}{

        $\btheta_{b, s+1} \gets \btheta_{b, s} + \eta \, (\bpi_{b, 0} \odot \bmu - J(\bpi_{b, 0}) \bpi_{b, s})$.

    }

    Set $\btheta_{b+1, 0} \gets \btheta_{b, S}$.
}

\end{algorithm}

\end{minipage}
\end{center}

\subsection{Global convergence to the optimal policy}
\label{subsec:theory_global_convergence}

The following theorem guarantees global convergence of \res to the optimal policy $\bpi^{\star} = e_1$, namely probability 1 on the optimal action and 0 elsewhere.

\begin{theorem}
\label{thm:global}
Suppose that $\btheta_{0, 0} \in \R^K$ is finite, and the learning rate $\eta > 0$ satisfies $\eta \cdot \mumax < 4$. 
Then for any fixed integer $S \ge 1$, \res converges to the optimal policy asymptotically:
\begin{align*}
\bpi_{b, s} \rightarrow e_1 \quad \text{for all} \quad 0 \le s \le S \quad \text{as} \quad b \rightarrow \infty.
\end{align*}
\end{theorem}

\begin{proof}[Proof sketch]
The proof combines improvement of the surrogate objective with bounds on changes in the logits. The step-size condition and smoothness ensure that $J_b \coloneqq J(\pi_{b,0})$ is nondecreasing in $b$, so it converges to some $J_\infty\leq\mu(1)$. The remaining work is to rule out convergence to a suboptimal action. 
To prove this by contradiction, suppose that $J_\infty<\mu(1)$. For fixed $S$, logit changes during each stage tend to zero. Actions with rewards above $J_\infty$ eventually have increasing logits, and those below it have decreasing logits. The estimate $\|g_{b,0}\|_2\leq C(J_\infty-J_b)$ implies $\sum_b(J_\infty-J_b)=\infty$, since a finite sum would make all logits converge to finite values and prevent the optimal action's gradient from tending to zero. Keeping $J_b\leq J_\infty$ therefore requires an action $a$ with $\mu(a)<J_\infty$ and $\sum_b\pi_{b,0}(a)=\infty$, therefore its logit tends to $-\infty$. Since the optimal logit is bounded below, the optimal action eventually has greater probability, which implies $\sum_b\pi_{b,0}(1)=\infty$ and its logit tends to $+\infty$. The probability of every action with $\mu(a)<J_\infty$, divided by the optimal action's probability, then tends to zero. These actions can no longer offset the optimal action's positive contribution to $J_b-J_\infty$, giving $J_b>J_\infty$, which is a contradiction. Therefore $J_\infty$ must be $\mu(1) = \mumax$, which means convergence to the optimal policy. See Appendix~\ref{app:global-convergence} for the full proof.
\end{proof}

\subsection{Two-sided bounds for convergence rates}
\label{subsec:theory_convergence_rates}

We next characterize the convergence rate of \res and its dependence on $S$. 
Denote the suboptimality gaps of policies by
\begin{align*}
\rg(\pi) \coloneqq \mumax-J(\pi), \qquad \ir_b \coloneqq \rg(\pi_{b,0}),
\end{align*}
and define the surrogate loss function $L$ as
\begin{align*}
    L(\pi_{b,s};\pi_{b,0}) \coloneqq \E_{a \sim \bpi_{b, 0}} \mu(a)\log\pi_{b,s}(a).
\end{align*}
Our analysis critically relies on the following lemma that quantifies progress within each inner loop.

\begin{lemma}[Progress and suboptimality gaps within an inner loop]\label{lem:inner-loop-progress}
Consider the $b$-th outer iteration of Algorithm~\ref{alg:res_softmax}.
Suppose that $\eta \cdot \mumax < 4$. For every $1\leq s\leq S$, we have
\begin{align}
 \lambda\big(\pi_{b,0}(1)\big)\min\{\eta s\ir_b^2,\ir_b\}
 &\leq L(\pi_{b,s};\pi_{b,0})-L(\pi_{b,0};\pi_{b,0}) \nonumber\\
 &\leq J(\pi_{b,s})-J(\pi_{b,0})
 \leq A\min\{\eta s\ir_b^2,\ir_b\},
 \label{eq:inner loop-progress}\\
 \text{and} \quad \rho\ir_b &\leq \rg(\pi_{b,s}) \leq\ir_b,
 \label{eq:inner loop-residual}
\end{align}
where the specific forms of $\lambda(\cdot)$, $\rho$ and $A$ can be found in Appendix \ref{app:conv-rate-lemma}.
\end{lemma}

\begin{proof}[Proof sketch.]
We focus on explaining the proof idea of the most important lower bound in Eq.~\eqref{eq:inner loop-progress}, and proofs of the remaining results can be found in Appendix~\ref{app:conv-rate-lemma}. It is known that for on-policy softmax policy gradient, smoothness and the non-uniform \L{}ojasiewicz inequality give an $\Omega(\eta\ir_b^2)$ improvement per
step~\citep{mei2020global}. For fixed inner loop step $s$, when $\eta s\ir_b\leq1$, a small suboptimality gap $\ir_b$ limits how far the current policy can move over the inner loop. Therefore, the accumulated progress is $\Omega(s\eta \ir_b^2)$ in this inner loop. This also explains the on-policy order of progress near convergence. But when $\eta s\ir_b>1$, the gradient can be very different from on-policy gradient during the inner loop. Nonetheless, the optimal action gradient component remains
at least half of its initial value for about $1/(\eta\ir_b)$ steps. Combining this with the on-policy result, we see that the progress in this inner loop is at least $\Omega(\eta\ir_b^2 \cdot 1/(\eta \ir_b)) = \Omega(\ir_b)$.
\end{proof}

Eq.~\eqref{eq:inner loop-progress} in the above lemma suggests that 
$J(\pi_{b,s})-J(\pi_{b,0})=\Theta\!\left(\min\{\eta s\ir_b^2,\ir_b\}\right)$. 
With the progress of the optimization, decreasing $\ir_b$ changes this scale
from $\Theta(\ir_b)$ to $\Theta(\eta s\ir_b^2)$. This change essentially causes the two time scales that we will soon describe in
the main theorems below. 
The reason behind the lower bound in Eq.~\eqref{eq:inner loop-residual} is that, increasing the likelihood of the reward-weighted rollout distribution prevents an arbitrarily large reduction of the suboptimality gap within the same inner loop.

Building on Lemma~\ref{lem:inner-loop-progress}, we can now establish upper and lower bounds for the convergence rates of \res in Theorems~\ref{thm:rate-upper} and~\ref{thm:rate-lower} respectively.
We introduce a new notation $p_{\min}:=\inf_{b\geq0}\pi_{b,0}(1)>0$, whose positivity\footnote{
For an arbitrary finite initialization, $p_{\min}>0$ but its value depends on
the optimization trajectory and could be very small \citep{mei2020global}. 
Fortunately, if we impose stronger conditions on the initial policy, the lower bound of $p_{\min}$ can be significantly improved. 
For example, if $\eta$ satisfies $0<\eta\mumax\leq2$ and $\pi_{0,0}(1)\geq\pi_{0,0}(a)$ for every $a>1$, then this ordering will be preserved at every inner iterate, as shown in Appendix~\ref{app:ordered-proof}. Hence $p_{\min}\geq1/K$ in this case, and the progress coefficient can be bounded below by $\lambda(1/K)$. 
} originates from Theorem~\ref{thm:global}.
In the following theorems, we use the notation $t = b S + s$, where $0\leq s\leq S$, for the total number of gradient steps needed to achieve $\pi_{b, s}$ in \res.

\begin{theorem}[Upper Bound]
\label{thm:rate-upper}
Suppose that $\eta \cdot \mumax < 4$.
There exists a constant $c>0$ (depending only on $K$, $\mu$, $\eta$, and $p_{\min}$) such that, if we define
\begin{equation}
 \btran =\begin{cases}
 0 &\text{if} \quad \eta S\ir_0\leq 1,\\[2pt]
 \displaystyle\left\lceil\frac{\log(\eta S\ir_0)}{\log(1+c)}\right\rceil
 &\text{if} \quad \eta S\ir_0>1,
 \end{cases}
\end{equation}
then for $0 \le b \le \btran$, we have
\begin{equation}
\rg(\pi_{b, s}) \leq\ir_0 (1+c)^{-b},
\end{equation}
whereas for every $b \geq \btran$, we have
\begin{equation}
 \rg(\pi_{b, s})\leq
 \frac1{\ir_{\btran}^{-1}+c\eta(t-\ttran)}, \quad \text{where} \quad 
 t = b S + s, \quad \ttran = \btran S.
\end{equation}

\end{theorem}

\begin{theorem}[Lower Bound]\label{thm:rate-lower}
Suppose that $\eta \cdot \mumax < 4$.
Let $A$ and $\rho\in(0,1)$ be the constants in Lemma~\ref{lem:inner-loop-progress}, and define
\begin{equation}
 \btran^{\prime} = \begin{cases}
 0 &\text{if} \quad \eta S\ir_0\leq (1-\rho)/A,\\[2pt]
 \displaystyle\left\lceil\frac{\log(A\eta S\ir_0/(1-\rho))}{\log(1/\rho)}\right\rceil
 &\text{if} \quad \eta S\ir_0>(1-\rho)/A,
 \end{cases}
\end{equation}
then for $0 \le b \le \btran^{\prime}$, we have
\begin{equation}
\rg(\pi_{b, S}) \geq\ir_0 \rho^{b+1},
\end{equation}
whereas for every $b \geq \btran^{\prime}$, we have
\begin{equation}
\rg(\pi_{b, s})\geq
\frac{1}{\ir_0^{-1}\rho^{-\btran^{\prime}}+(A/\rho)\eta(t-\ttran^{\prime})},
\quad \text{where} \quad 
t = b S + s, \quad \ttran^{\prime} = \btran^{\prime} S.
\end{equation}
\end{theorem}

\begin{proof}[Proof sketch]
Using Lemma~\ref{lem:inner-loop-progress} and dividing Eq.~\eqref{eq:inner loop-progress} therein by $\ir_b\rg(\pi_{b,s})$, we have
\begin{equation*}
\lambda(p_{\min})\min\{\eta s,\frac1{\ir_{b}}\} \leq \lambda\big(\pi_{b,0}(1)\big)\min\{\eta s,\frac1{\ir_{b}}\}
\leq \frac1{\rg(\pi_{b,s})}-\frac1{\ir_{b}}
\leq
\min\left\{\frac{A}{\rho}\eta s,
\frac{1-\rho}{\rho}\frac1{\ir_{b}}\right\}.
\end{equation*}
Denote $a_b \coloneqq 1 / \ir_b$.
For $s=S$, replacing each bound by equality gives two sequences of the form $a_{b+1}=a_b+\min\{x\eta S,ya_b\}$, starting at $a_0=1/\ir_0$.
The coefficients $x$ and $y$ come from the corresponding side of the inequality.
This update is increasing in $a_b$, so the two sequences remain lower and upper bounds on $1/\ir_b$ at every inner loop.
When $ya_b<x\eta S$, the update is $a_{b+1}=(1+y)a_b$.
Once $ya_b\geq x\eta S$, it becomes $a_{b+1}=a_b+x\eta S$ and remains in this case.
See Appendix~\ref{app:rate-profile} for the complete proof.
\end{proof}

\begin{remark}
Several previous works studied similar problems but with different theoretical techniques. 
\citet{mei2020global} lower bound improvement by a multiple of the squared gap using smoothness and a non-uniform \L{}ojasiewicz inequality;
such a bound is not guaranteed to hold for every step of an inner loop in the \res algorithm that we study, since its gradient update $g = q\odot\mu-J(q)\pi$ (where $q$ denotes the rollout distribution) can vanish. %
\citet{liu2024elementaryanalysispolicygradient} use the nonnegativity of
$(u(a)-u(a'))(e^{\eta u(a)}-e^{\eta u(a')})$,
where $u=\nabla_\lo J$; for \res, the corresponding terms are $(u(a)-u(a'))(e^{\eta g(a)}-e^{\eta g(a')})$, and these two factors need not have the same sign. Some works consider stage-wise algorithms similar to \res, where each inner loop uses a fixed rollout distribution $q$:
\citet{arnal2025asymmetric} identify the inner loop convergence limit through a Lyapunov function, and then iterate this limiting map across outer iterations,
while \citet{strupl2022rwr} use exact weighted-likelihood maximization, nonnegative action-value variance, and probability ratios to prove improvement and convergence.
For bandits with positive rewards, both yield the updated rollout distribution
$q^+=q\odot\mu/J(q)$, but \res with a finite $S$ value need not obey this identity. 
Adapting their arguments to \res would require additional control of how incomplete fitting within a stage changes these ratios. 
Our proposed analysis instead estimates improvement over a finite-step stage in terms of the current suboptimality gap, including cases when the gap approaches zero and when it does not.
\end{remark}

\subsection{Benefits of off-policyness}
\label{subsec:theory_benefits_of_off_policyness}

While the previous theorems present $\Theta(1/t)$ asymptotic convergence rates regardless of the value of $S$,
there exist other terms in the bounds that can play a crucial role in the overall speed of convergence.
In the following, we identify settings where, perhaps surprisingly, \res with a suitably large $S$ can achieve much faster convergence than on-policy \reone.

Assume $K\geq3$ and consider the following assumptions:
\begin{align*}
    \mumax = \mu(1) > \mu(2) > \mu(a) \ge \mumin > 0, \quad 3 \le a \le K;
    \qquad
    \Delta = \mu(1) - \mu(2) > 0.
\end{align*}
Fix weights $w(a)>0$ for $2 \le a \le K$ with $\sum_{a=2}^K w(a)=1$, and initialize \res with the following policy $\pi_{0, 0}$:
\begin{equation}
\pi_{0,0}(1)=\pb > 0,\quad
\pi_{0,0}(a)=(1-\pb)w(a), \quad 2 \le a \le K.
\label{eq:comparison-initialization}
\end{equation}
Now we consider a family of initial conditions specified by the optimal-action probability $\pb \in (0,1)$, while fixing other factors like $K$, $\mu$, $w$, and $\eta$ satisfying $0<\eta\mumax<4$. 
For any target accuracy $0<\epsilon<\Delta/2$, we define $\Teps(S)$ as the number of gradient steps required by \res to achieve it:
\begin{equation}
\Teps(S):=\inf\{bS:b \ge 0,\ \rg(\pi_{b,0})\leq\epsilon\}.
\label{eq:comparison-endpoint-hitting}
\end{equation}
Focusing on asymptotic $x \rightarrow 0$, the theorem below demonstrates provable benefits of off-policyness.
\begin{theorem}
\label{thm:benefits_off_policyness}
Under the conditions stated above, we have
\begin{equation*}
\Teps(1)=\Omega\!\left(\pb^{-(K-1)}\right).
\end{equation*}
Moreover, there exists an integer $S_\pb=\Theta(\pb^{-1})$, for which
\begin{equation*}
\Teps(S_\pb)=O\!\left(\pb^{-1}\log(1/\pb)\right),
\quad \text{and thus} \quad
\frac{\Teps(S_\pb)}{\Teps(1)}\longrightarrow 0 \quad \text{as} \quad x \rightarrow 0.
\end{equation*}
The implicit constants in these bounds may depend on $K,\mu,w,\eta,\epsilon$, but not on $\pb$.
\end{theorem}

We provide an intuitive explanation below.
On-policy policy gradient follows the direction of steepest first-order increase in expected reward,
yet this local improvement might substantially reduce the optimal action's probability. 
Once trapped near a suboptimal vertex of the probability simplex, the small probability of the optimal action could suppress its gradient, making recovery very slow. This slow transient behavior of softmax policy gradient has also been established theoretically and observed empirically by previous works
\citep{mei2020gravity,mei2026delightful}.
In contrast, \res fixes the rollout distribution within each stage and takes gradient steps to reduce the forward KL divergence from the reward-weighted target to the current policy. 
Reward weighting assigns the largest multiplicative weight to the optimal action, so once the KL divergence is sufficiently small, its probability increases relative to its value at the start of the stage.
Choosing a suitable $S$ for sufficiently accurate fitting therefore
ensures progress in the optimal action's probability after each inner loop.
This property holds for every full-support policy, including policies near a suboptimal vertex. Therefore, even if \res has slower progress than on-policy \reone early in training, it can maintain steady progress in increasing the probability of the optimal action, avoiding the stagnation that may occur with \reone and thus achieving faster convergence eventually.

\begin{proof}[Proof sketch.]
For \reone, probability-ratio identities and conservation of the logit sum show that the optimal-action probability becomes $O(x^{K-1})$ while most probability concentrates on action $2$.
Controlling the time spent in this region yields the $\Omega(x^{-(K-1)})$ lower bound.
For \res, each reward-weighted target increases the log-odds of the optimal  action by at least a fixed positive amount.
Its logits differ from the stage's initial logits only by reward-dependent shifts, and standard optimization analysis gives an $O(1/S)$
KL fitting error bound for an $S$-step stage, with a constant term independent of $x$.
A binary-KL bound shows that an error below a sufficiently small multiple of $x$ preserves a fixed fraction of this log-odds gain.
Thus $S_x=\Theta(x^{-1})$ suffices to ensure that the optimal-action probability at the end of a stage remains at least $x$ throughout. 
Accumulating the fixed log-odds gains takes $O(\log(1/x))$ stages, leading to $O(x^{-1}\log(1/x))$ total gradient steps.
See Appendix~\ref{app:comparison} for the complete proof.
\end{proof}

\subsection{Numerical validation}
\label{subsec:numerical_validation}

We first conduct experiments to verify the theoretical findings in Theorems~\ref{thm:rate-upper} and~\ref{thm:rate-lower}, namely the convergence rates of \res.
The empirical results are shown in Figure~\ref{fig:exp_convergence_rates}, whose concrete settings are explained in the caption.
In the first setting, we observe that when initialized at a strong policy,
\res converges to the optimal policy at a $\Theta(1/t)$ rate (where $t$ is the number of gradient steps) after a short burn-in phase;
in particular, the learning curves of \res with different $S$ values mostly overlap when using $t$ as the X-axis.
In the second setting where the initial policy is relatively weaker (though still dominated by the optimal action) and the learning rate $\eta$ is larger,
\res with a larger $S$ exhibits a longer burn-in phase and initially lags behind, but eventually catches up with on-policy \reone and achieves the same $\Theta(1/t)$ convergence rate,
as predicted by our main theorems.

We further conduct experiments to verify Theorem~\ref{thm:benefits_off_policyness}, namely the benefits of off-policyness in accelerating the convergence of \res when the initial policy has a small optimal-action probability.
The empirical results are shown in Figure~\ref{fig:exp_on_vs_off_policy}, whose concrete settings are explained in the caption.
In the first plot, we observe that the suboptimality gap of on-policy \reone decays quickly at the beginning, but then plateaus at $0.3$ --- the suboptimality gap of the second-best action --- for a long period;
in contrast, \res with a large $S=512$ makes steady progress towards the global optimum and eventually reaches a suboptimality gap close to zero.
The second plot visualizes the optimization trajectories of \res with different $S$ values in a minimal 3-action setting;
it clearly demonstrates how \reone takes a long detour in its trajectory towards the optimal policy, whereas \res with larger $S$ avoids the trap and achieves faster convergence to the global optimum.

\clearpage
\newpage

\begin{figure}
    \centering
    \includegraphics[width=0.45\linewidth]{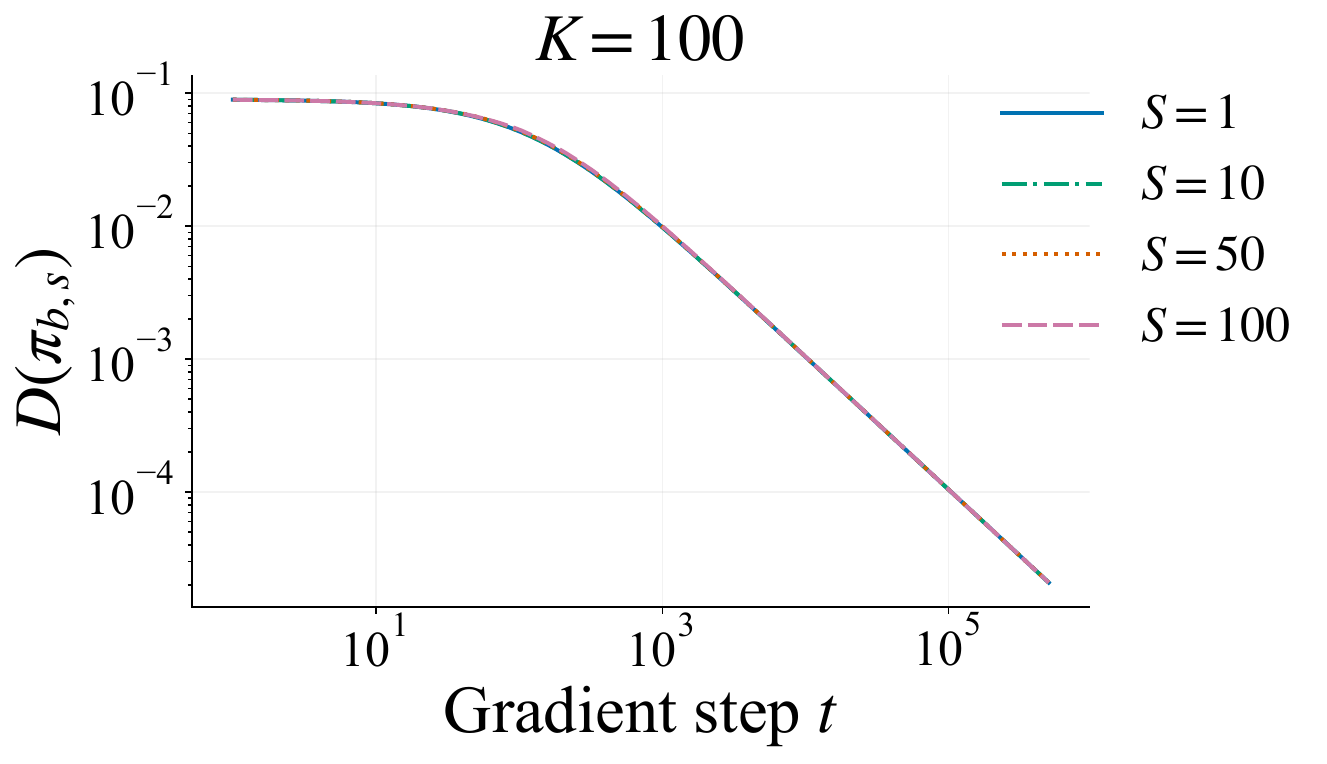} 
    \includegraphics[width=0.45\linewidth]{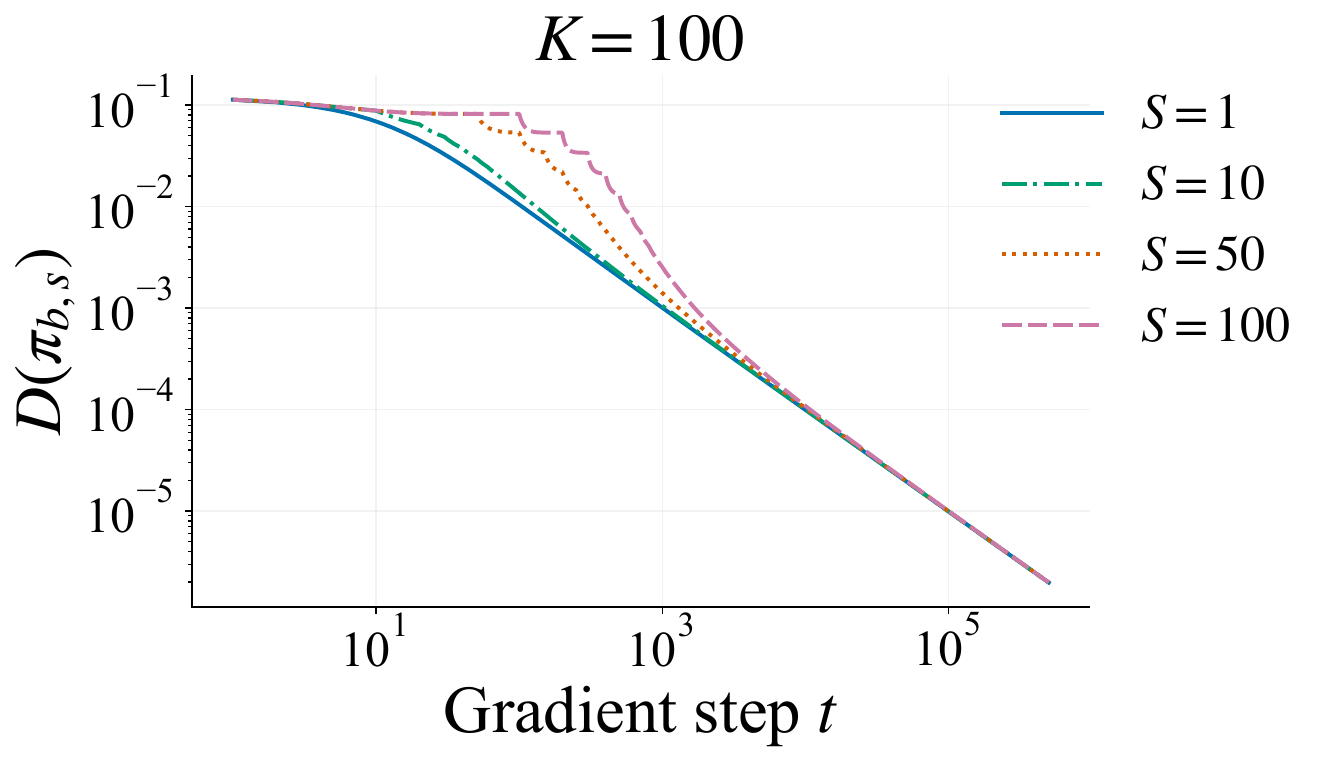}
    \caption{Empirical validation of \res convergence rates in two concrete settings, both with $K=100$.
    \textbf{Left: } the reward-mean vector is $\mu = [1, 0.1, \dots, 0.1]$, the initial policy $\pi_{0,0}$ has probability $0.9$ on the optimal action and equal probabilities on the remaining actions, and the learning rate $\eta = 0.095$.
    \textbf{Right: } $\mu = [1, 0.6, \dots, 0.6]$, $\pi_{0,0}$ has probability $0.7$ on the optimal action and equal probabilities on the remaining actions, and $\eta = 1$.
    }
    \label{fig:exp_convergence_rates}
\end{figure}

\begin{figure}
    \centering
    \includegraphics[width=0.45\linewidth]{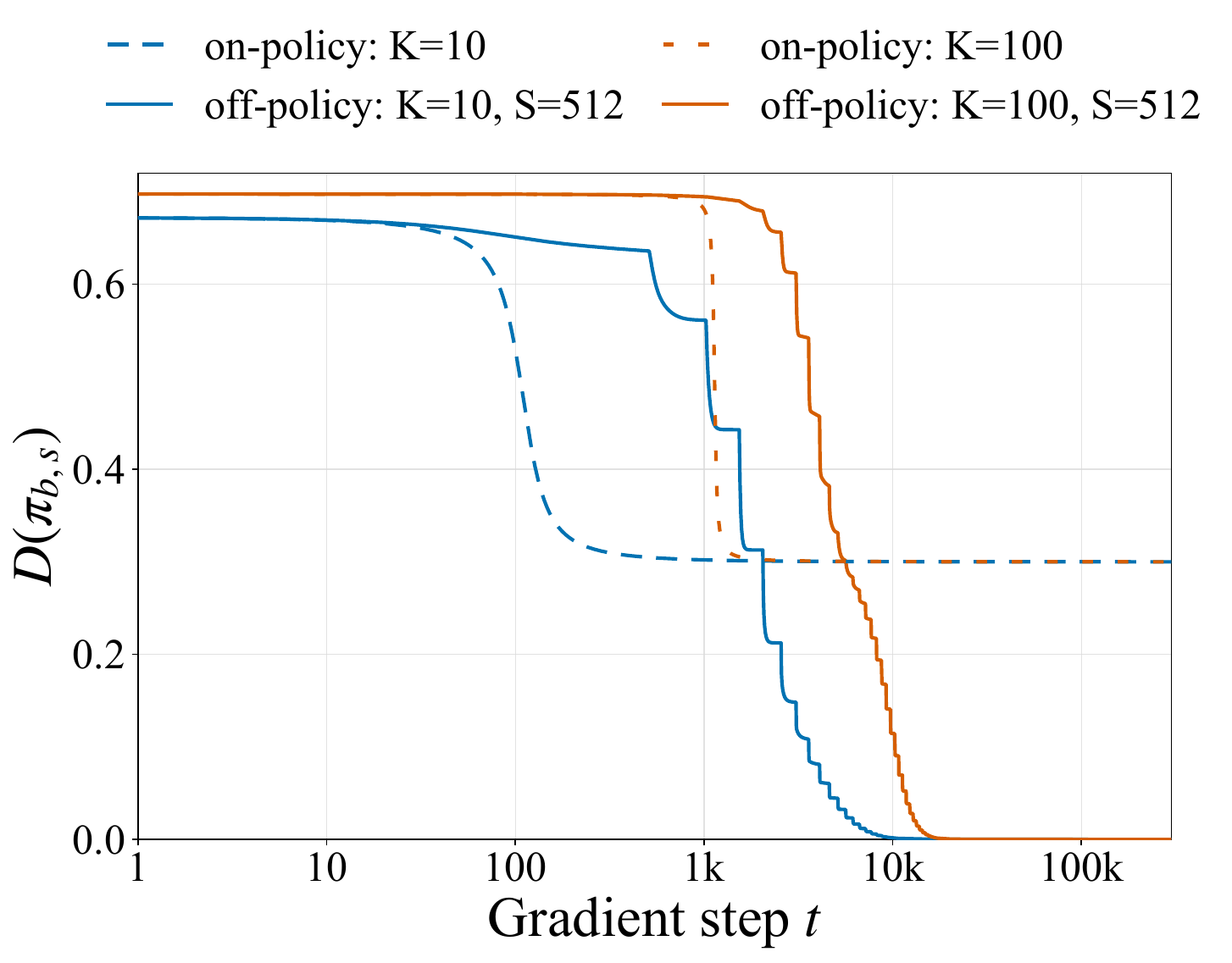} 
    \includegraphics[width=0.42\linewidth]{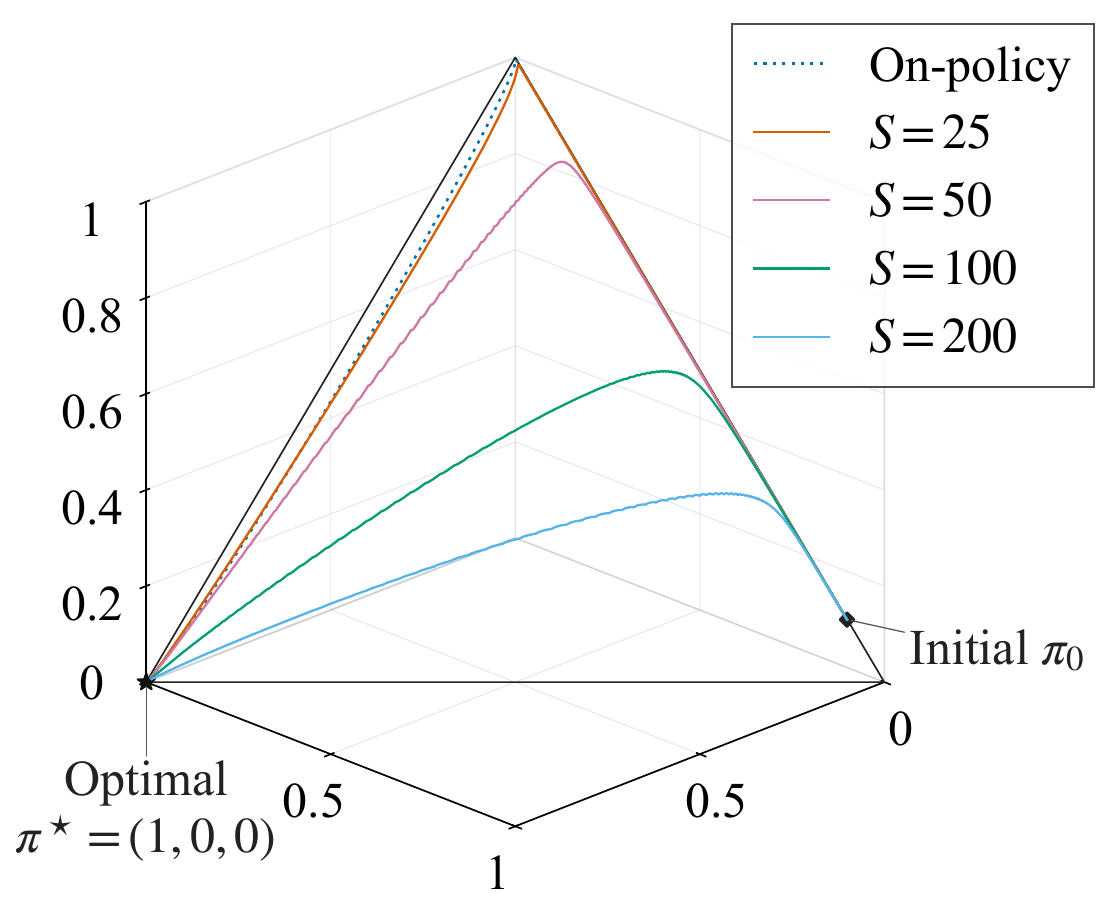}
    \caption{Empirical validation of the benefits of off-policyness.
    \textbf{Left: } $K \in \{10, 100\}$, the reward-mean vector is $\mu = [1, 0.7, 0.3, 0.3, \dots]$, the initial policy $\pi_{0, 0}$ satisfies $\pi_{0, 0}(1) = 0.1 / K$ and $\pi_{0, 0}(a) \propto e^{-2 \mu(a)}$ for $2 \le a \le K$.
    \textbf{Right: } a minimal setting for visualizing the optimization trajectories of \res, with $K=3$, $\mu = [1, 0.9, 0.89]$, and $\pi_{0,0} = [0.0005, 0.1, 0.8995]$. The learning rate $\eta = 0.5$ for both experiments.
    }
    \label{fig:exp_on_vs_off_policy}
\end{figure}

%% file: sections/related_work.tex
\section{Related work}
\label{sec:related_work}

\paragraph{Reward-guided self-training.}

Reward-guided self-training has been used to improve LLM reasoning and alignment by training on model-generated outputs that are weighted or selected based on task feedback. STaR learns from rationales that yield correct answers, RAFT selects outputs for training by reward ranking, and ReST-EM uses binary feedback in an expectation-maximization procedure \citep{zelikman2022star,dong2023raft,singh2024humandatascalingselftraining}. \citet{ghosh2020operator} derive a policy-improvement bound for reward-weighted likelihood fitting that does not require exact optimization.
In tabular settings, reward-weighted regression and zero-baseline AsymRE --- both of which have a stage-wise structure akin to \res \ --- have been shown to converge globally from any full-support initial policy when the surrogate objective within each stage is fitted exactly or to convergence \citep{strupl2022rwr,arnal2025asymmetric}, but their results do not cover finite-step fitting; in comparison, our convergence results in this work cover the full range $S \ge 1$.
Methods with finite-step inner loops include iw-SFT, which periodically reweights a fixed curated dataset and yields a reward-weighted surrogate loss before clipping or smoothing operations \citep{qin2025curated}.
\citet{russo2026success} formulates finite-step fitting of success-conditioned targets, but does not establish convergence of the resulting iterative procedure.
All these works do not establish global convergence or rates measured in total gradient steps when each stage performs a fixed finite number $S$ of updates.

\paragraph{Softmax policy gradient.}

Global convergence and $O(1/T)$ rates have been established for exact softmax policy gradient \citep{mei2020global,liu2024elementaryanalysispolicygradient,lu2024towards}; our results in Theorems~\ref{thm:global},~\ref{thm:rate-upper} and~\ref{thm:rate-lower} extend these conclusions to \res with $S \ge 1$.
Almost-sure convergence has also been established for stochastic REINFORCE in bandits and finite-horizon tabular MDPs \citep{mei2023stochastic,robertson2025reinforce}.
Lower bounds nevertheless show slow escape from unfavorable initialization in bandits and exponential iteration complexity on particular discounted MDPs \citep{mei2020gravity,li2023exponential}. \citet{mei2020gravity} address this issue by changing the
policy parameterization, while \citet{mei2026delightful}
use gradient gating to accelerate escape from suboptimal corners of the probability simplex. 
Our results in Theorem~\ref{thm:benefits_off_policyness} show that an appropriate degree of off-policyness can also reduce this delay without further algorithmic modifications.

\paragraph{Off-policy algorithms.}

Global convergence guarantees for dedicated off-policy RL algorithms allow state-distribution mismatch \citep{laroche2021jekyll,zhang2022offpolicy}.
They retain current-policy action weights, directly or through importance-sampling ratios, whereas \res uses the stale action weights without
importance-sampling correction. FMA-PG provides policy-improvement guarantees for finite-step surrogate optimization within each inner loop, while SPMA's convergence bound includes an additive term for errors due to approximate fitting \citep{vaswani2022surrogates,asad2025spma}.
\citet{dai2026nonasymptoticglobalconvergenceppoclip} prove global linear convergence of multi-step PPO-Clip with KL regularization under a bound on the sum of inner step sizes.
All these guarantees do not directly yield global convergence
rates for the uncorrected, unregularized reward-weighted update rule in \res with a fixed $S$ value;
our work fills in this gap of the literature.

%% file: sections/conclusion.tex
\section{Limitations and future work}
\label{sec:conclusion}

This work has investigated the learning dynamics of \res --- probably the simplest possible algorithm for fine-tuning on self-generated and reward-weighted data --- with in-depth convergence analysis for multi-arm bandits and softmax policies.
Future work may extend the theoretical study to broader settings, such as contextual bandits, Markov decision processes, policies with different parameterizations, or other learning algorithms.
Moreover, our analysis focuses on convergence properties and considers the infinite-sample limit of gradient dynamics; future work may investigate sample complexities in finite-sample settings.
In terms of empirical work, it remains open to verify which parts of our results hold true in realistic scenarios, see if our theoretical results can inspire better practice of reward-guided self-training, and identify gaps between theory and practice that require further research.

%% file: appendix/convergence_proofs.tex
\section{Proofs for Section~\ref{subsec:theory_global_convergence}}
\label{app:global-convergence}

\begin{lemma}\label{lem:surrogate obj grad}
    Consider the surrogate objective function
\begin{equation*}
L(\pi_{b,s};\pi_{b,0})
=\sum_{a\in\mathcal A}\pi_{b,0}(a)\mu(a)\log\pi_{b,s}(a).
\end{equation*}
With $\pi_{b,0}$ fixed, the gradient with respect to the logit parameter $\theta$, evaluated at $\theta=\theta_{b,s}$, satisfies
\begin{equation*}
\begin{aligned}
g_{b,s} \coloneqq \left.\nabla_\theta L(\pi_\theta;\pi_{b,0})\right|_{\theta=\theta_{b,s}}=\pi_{b,0}\odot\mu-J(\pi_{b,0})\pi_{b,s}.
\end{aligned}
\end{equation*}
\end{lemma}

\begin{proof}
Note that\begin{equation*}
\begin{aligned}
\left.\nabla_\theta L(\pi_\theta;\pi_{b,0})\right|_{\theta=\theta_{b,s}}
&=\sum_{a\in\mathcal A}\pi_{b,0}(a)\mu(a)
\frac{\left.\nabla_\theta\pi_\theta(a)\right|_{\theta=\theta_{b,s}}}{\pi_{b,s}(a)}\\
&=\sum_{a\in\mathcal A}\pi_{b,0}(a)\mu(a)
\frac{\pi_{b,s}(a)e_a-\pi_{b,s}(a)\pi_{b,s}}{\pi_{b,s}(a)}\\
&=\sum_{a\in\mathcal A}\pi_{b,0}(a)\mu(a)e_a
-\sum_{a\in\mathcal A}\pi_{b,0}(a)\mu(a)\pi_{b,s}\\
&=\pi_{b,0}\odot\mu-J(\pi_{b,0})\pi_{b,s}=g_{b,s}.
\end{aligned}
\end{equation*}
\end{proof}

\begin{lemma}\label{lem:lsmooth}
$\nabla_\theta L(\pi_\theta;\pi_{b,0})$ has a Lipschitz constant no greater than $J(\pi_{b,0})/2$.
\end{lemma}
\begin{proof}
    Define
\nopagebreak[4]
\begin{equation*}
H(\pi)=\operatorname{diag}(\pi)-\pi\pi^\top.
\end{equation*}
This matrix is positive semidefinite. By Gershgorin's theorem, its eigenvalues are bounded above by $\max_a2\pi(a)(1-\pi(a))\le1/2$, and hence
\begin{equation}
\label{eq:thm1-policy-hessian-bound}
0\preceq H(\pi)\preceq I/2.
\end{equation}
We also have
\begin{equation}
\label{eq:thm1-surrogate-hessian}
\nabla_\theta^2L(\pi_\theta;\pi_{b,0})
=- J(\pi_{b,0}) H(\pi_\theta).
\end{equation}
Thus $\nabla_\theta L(\pi_\theta;\pi_{b,0})$ has a Lipschitz constant no greater than $J(\pi_{b,0})/2$.
\end{proof}

\begin{lemma}\label{lem:grad order}
Suppose $\eta \mumax<4$. Consider how the logit of each action evolves under this assumption. We have
\begin{equation*}
\|g_{b,s}\|_2\le\|g_{b,0}\|_2.
\end{equation*}
\end{lemma}

\begin{proof}
Fix $b$. Note that the surrogate objective is concave and can be written as
\begin{equation*}
L(\pi_\theta;\pi_{b,0})
=\sum_{a\in\mathcal A}\pi_{b,0}(a)\mu(a)\theta(a) - J(\pi_{b,0})\log\sum_{a\in\mathcal A}e^{\theta(a)}.
\end{equation*}
The Hessian bound from Eq.~\eqref{eq:thm1-policy-hessian-bound} and Eq.~\eqref{eq:thm1-surrogate-hessian} shows that $\nabla L(\pi_\theta;\pi_{b,0})$ has a Lipschitz constant of at most $J(\pi_{b,0})/2$. For $0\le s<S$, by the standard co-coercivity inequality \citep{nesterov2013introductory},
\begin{equation*}
\begin{aligned}
&\left\langle
\nabla_\lo L(\pi_{b,s};\pi_{b,0}) - \nabla_\lo L(\pi_{b,s+1};\pi_{b,0}),
\theta_{b,s+1}-\theta_{b,s}
\right\rangle\\
&\qquad\ge\frac{2}{J(\pi_{b,0})}
\left\|\nabla_\lo L(\pi_{b,s};\pi_{b,0}) - \nabla_\lo L(\pi_{b,s+1};\pi_{b,0})\right\|_2^2.
\end{aligned}
\end{equation*}
Substituting $\nabla_\lo L(\pi_{b,s};\pi_{b,0})=g_{b,s}$ and $\theta_{b,s+1}-\theta_{b,s}=\eta g_{b,s}$ gives
\begin{equation*}
\eta\langle g_{b,s}-g_{b,s+1},g_{b,s}\rangle
\ge\frac{2}{J(\pi_{b,0})}\|g_{b,s+1}-g_{b,s}\|_2^2.
\end{equation*}
Therefore, since $\eta J(\pi_{b,0})<4$,
\begin{equation*}
\begin{aligned}
\|g_{b,s+1}\|_2^2
&=\|g_{b,s}\|_2^2
+2\langle g_{b,s+1}-g_{b,s},g_{b,s}\rangle
+\|g_{b,s+1}-g_{b,s}\|_2^2\\
&\le\|g_{b,s}\|_2^2
-\left(\frac{4}{\eta J(\pi_{b,0})}-1\right)
\|g_{b,s+1}-g_{b,s}\|_2^2\\
&\le\|g_{b,s}\|_2^2.
\end{aligned}
\end{equation*}
It follows that
\begin{equation*}
\|g_{b,s}\|_2\le\|g_{b,0}\|_2,
\end{equation*}
which completes our proof.
\end{proof}

\paragraph{Proof of Theorem ~\ref{thm:global}.}

Using lemma \ref{lem:lsmooth}, we have
\begin{equation}
\label{eq:surrogate-ascent}
L(\pi_{b,s+1};\pi_{b,0})-L(\pi_{b,s};\pi_{b,0})
\ge\eta\left(1-\frac{\eta J(\pi_{b,0})}{4}\right)\|g_{b,s}\|_2^2.
\end{equation}
Action $1$ is uniquely optimal and rewards are nonnegative, so $\mumax>0$. Every policy coordinate remains strictly positive after finitely many updates, and hence $J(\pi_{b,0})\ge \mumax \pi_{b,0}(1)>0$. The right side is nonnegative because $0<J(\pi_{b,0})\le \mumax$ and $\eta \mumax<4$. For $0\le s\le S$, using the inequality $\log x\le x-1$ we have
\begin{equation}
\label{eq:thm1-surrogate-value-bound}
\begin{aligned}
L(\pi_{b,s};\pi_{b,0})-L(\pi_{b,0};\pi_{b,0})
&=\sum_{a\in\mathcal A}\pi_{b,0}(a)\mu(a)
\log\frac{\pi_{b,s}(a)}{\pi_{b,0}(a)}\\
&\le\sum_{a\in\mathcal A}\pi_{b,0}(a)\mu(a)
\left(\frac{\pi_{b,s}(a)}{\pi_{b,0}(a)}-1\right)\\
&=J(\pi_{b,s})-J(\pi_{b,0}).
\end{aligned}
\end{equation}
Hence $J(\pi_{b,s})\ge J(\pi_{b,0})$. Taking $s=S$ gives $J(\pi_{b+1,0})\ge J(\pi_{b,0})$. Since $J(\pi_{b,0})\le \mumax$, there exists $J_\infty\le \mumax$ such that $J(\pi_{b,0})\to J_\infty$. Summing Eq.~\eqref{eq:surrogate-ascent} over the updates in the inner loop and using Eq.~\eqref{eq:thm1-surrogate-value-bound} gives
\begin{equation}
\label{eq:thm1-round-value-increase}
J(\pi_{b+1,0})-J(\pi_{b,0})
\ge\eta\left(1-\frac{\eta \mumax}{4}\right)
\sum_{s=0}^{S-1}\|g_{b,s}\|_2^2.
\end{equation}
Summing over $b$ then we have 
\begin{equation*}
\sum_{b=0}^{\infty}\sum_{s=0}^{S-1}\|g_{b,s}\|_2^2
\le\frac{\mumax-J_0}{\eta(1-\eta \mumax/4)}<\infty.
\end{equation*}
In particular, $g_{b,0}\to0$. We have
\begin{equation*}
g_{b,0}(a)=\pi_{b,0}(a)(\mu(a)-J(\pi_{b,0})).
\end{equation*}
Thus, if $\mu(a)\ne J_\infty$, then $\pi_{b,0}(a)\to0$. Define
\nopagebreak[4]
\begin{equation*}
\mathcal{M}=\{a\in\mathcal A:\mu(a)=J_\infty\}.
\end{equation*}
Since the probabilities sum to $1$, we have $\mathcal{M}\ne\varnothing$ and
\begin{equation}
\label{eq:thm1-limiting-set-mass}
\sum_{a\in \mathcal{M}}\pi_{b,0}(a)\longrightarrow1.
\end{equation}
We now prove $J_\infty=\mu(1)$. Suppose for contradiction that $J_\infty<\mu(1)$, and define
\begin{equation*}
\mathcal{M}_{+}=\{a:\mu(a)>J_\infty\},\qquad
\mathcal{M}_{-}=\{a:\mu(a)<J_\infty\}.
\end{equation*}
Then $1\in \mathcal{M}_{+}$. Since $J(\pi_{b,0})$ increases monotonically to $J_\infty$, we have $J(\pi_{b,0})\le J_\infty$ throughout.
By lemma \ref{lem:grad order}, we have
\begin{equation}
\label{eq:thm1-inner-stability}
\max_{0\le s\le S}\|\theta_{b,s}-\theta_{b,0}\|_2
\le\eta S\|g_{b,0}\|_2\longrightarrow0.
\end{equation}
By the definition of $\pi_\theta$,
\begin{equation}
\label{eq:thm1-inner-policy-ratio}
\frac{\pi_{b,s}(a)}{\pi_{b,0}(a)}
=\frac{\exp(\theta_{b,s}(a)-\theta_{b,0}(a))}
{\sum_{c\in\mathcal A}\pi_{b,0}(c)
\exp(\theta_{b,s}(c)-\theta_{b,0}(c))}
\longrightarrow1.
\end{equation}
Here the convergence is uniform over the finitely many actions $a$ and steps $0\le s\le S$. Thus
\begin{equation*}
\frac{g_{b,s}(a)}{\pi_{b,0}(a)}
=\mu(a)-J(\pi_{b,0})\frac{\pi_{b,s}(a)}{\pi_{b,0}(a)}
\longrightarrow\mu(a)-J_\infty.
\end{equation*}
Since the action set is finite, there exists $B_0$ such that, whenever $b\ge B_0$, the logit of each action in $\mathcal{M}_{+}$ strictly increases at every update of the inner loop, while the logit of each action in $\mathcal{M}_{-}$ strictly decreases. 
More precisely,
\begin{equation}
\label{eq:thm1-logit-direction-bounds}
\begin{aligned}
\theta_{b+1,0}(a)-\theta_{b,0}(a)
&\ge\frac{\eta S}{2}(\mu(a)-J_\infty)\pi_{b,0}(a),
&&a\in \mathcal{M}_{+},\\
\theta_{b+1,0}(a)-\theta_{b,0}(a)
&\le-\frac{\eta S}{2}(J_\infty-\mu(a))\pi_{b,0}(a),
&&a\in \mathcal{M}_{-}.
\end{aligned}
\end{equation}
The logit of action $1$ is nondecreasing from this point onward, so $\theta_{b,0}(1)$ is bounded below. At the same time, Eq.~\eqref{eq:thm1-limiting-set-mass} and $\pi_{b,0}(1)\to0$ give
\begin{equation*}
\sum_{a\in \mathcal{M}}\exp(\theta_{b,0}(a)-\theta_{b,0}(1))
=\frac{\sum_{a\in \mathcal{M}}\pi_{b,0}(a)}{\pi_{b,0}(1)}
\longrightarrow+\infty.
\end{equation*}
Hence
\nopagebreak[4]
\begin{equation}
\label{eq:thm1-limiting-set-logit-divergence}
\max_{a\in \mathcal{M}}\theta_{b,0}(a)\longrightarrow+\infty.
\end{equation}
We now bound the gradient norm in terms of
$J_\infty-J(\pi_{b,0})$. For notational simplicity, write 
\begin{align*}
h_b \coloneqq \|g_{b,0}\|_2.
\end{align*}
After the rollout policy is refreshed for the next stage, its initial gradient satisfies
\begin{equation*}
\begin{aligned}
g_{b+1,0}(a)
&=\pi_{b+1,0}(a)(\mu(a)-J(\pi_{b+1,0}))\\
&=\frac{\pi_{b+1,0}(a)}{\pi_{b,0}(a)}g_{b,0}(a)
-(J(\pi_{b+1,0})-J(\pi_{b,0}))\pi_{b+1,0}(a).
\end{aligned}
\end{equation*}
The gradient norm does not increase within the inner loop, so $\|\theta_{b+1,0}-\theta_{b,0}\|_2\le\eta S h_b$. The difference between the largest and smallest coordinates of a vector is at most $\sqrt2$ times its Euclidean norm. The ratio formula Eq.~\eqref{eq:thm1-inner-policy-ratio} for $\pi_\theta$ therefore gives
\begin{equation*}
\frac{\pi_{b+1,0}(a)}{\pi_{b,0}(a)}
\ge e^{-\sqrt2\eta S h_b}.
\end{equation*}
Combining the reverse triangle inequality with $\|\pi_{b+1,0}\|_2\le1$ and $J(\pi_{b+1,0})-J(\pi_{b,0})\ge0$ gives
\begin{equation*}
\begin{aligned}
h_{b+1}
&\ge\left(\sum_{a\in\mathcal A}
\left[\frac{\pi_{b+1,0}(a)}{\pi_{b,0}(a)}g_{b,0}(a)\right]^2
\right)^{1/2}
-(J(\pi_{b+1,0})-J(\pi_{b,0}))\|\pi_{b+1,0}\|_2\\
&\ge e^{-\sqrt2\eta S h_b}
\left(\sum_{a\in\mathcal A}g_{b,0}(a)^2\right)^{1/2}
-(J(\pi_{b+1,0})-J(\pi_{b,0}))\\
&=e^{-\sqrt2\eta S h_b}h_b-(J(\pi_{b+1,0})-J(\pi_{b,0})).
\end{aligned}
\end{equation*}
Using $1-e^{-x}\le x$, we obtain
\begin{equation*}
\begin{aligned}
h_b-h_{b+1}
&\le(1-e^{-\sqrt2\eta S h_b})h_b+(J(\pi_{b+1,0})-J(\pi_{b,0}))\\
&\le\sqrt2\eta S h_b^2+(J(\pi_{b+1,0})-J(\pi_{b,0})).
\end{aligned}
\end{equation*}
The lower bound Eq.~\eqref{eq:thm1-round-value-increase} on the increase in the objective gives
\begin{equation*}
\eta\left(1-\frac{\eta \mumax}{4}\right)h_b^2
\le J(\pi_{b+1,0})-J(\pi_{b,0}).
\end{equation*}
Then $h_b-h_{b+1}\le (1+\sqrt2 S/(1-\eta \mumax/4))(J(\pi_{b+1,0})-J(\pi_{b,0}))$. Summing from round $b$ to round $B$ gives
\begin{equation*}
h_b-h_{B+1}\le (1+\sqrt2 S/(1-\eta \mumax/4))(J(\pi_{B+1,0})-J(\pi_{b,0})).
\end{equation*}
Since $h_{B+1}\to0$ and $J(\pi_{B+1,0})\to J_\infty$, letting $B\to\infty$ yields
\begin{equation*}
\|g_{b,s}\|_2\le\|g_{b,0}\|_2
\le (1+\sqrt2 S/(1-\eta \mumax/4))(J_\infty-J(\pi_{b,0})),\qquad 0\le s<S.
\end{equation*}
If $\sum_b(J_\infty-J(\pi_{b,0}))<\infty$, then
\begin{equation*}
\begin{aligned}
\sum_{b=0}^{\infty}\|\theta_{b+1,0}-\theta_{b,0}\|_2
&\le\eta\sum_{b=0}^{\infty}\sum_{s=0}^{S-1}\|g_{b,s}\|_2\\
&\le\eta S(1+\sqrt2 S/(1-\eta \mumax/4))\sum_{b=0}^{\infty}(J_\infty-J(\pi_{b,0}))<\infty.
\end{aligned}
\end{equation*}
Every logit coordinate would then converge to a finite value, contradicting Eq.~\eqref{eq:thm1-limiting-set-logit-divergence}. Therefore
\begin{equation}
\label{eq:thm1-cumulative-value-gap}
\sum_{b=0}^{\infty}(J_\infty-J(\pi_{b,0}))=+\infty.
\end{equation}
We now use the increasing logits of actions with higher rewards and the decreasing logits of actions with lower rewards to derive a contradiction. Note that we have
\begin{equation*}
\begin{aligned}
0\le J_\infty-J(\pi_{b,0})
&=\sum_{a\in \mathcal{M}_{-}}(J_\infty-\mu(a))\pi_{b,0}(a)
-\sum_{a\in \mathcal{M}_{+}}(\mu(a)-J_\infty)\pi_{b,0}(a)\\
&\le\sum_{a\in \mathcal{M}_{-}}(J_\infty-\mu(a))\pi_{b,0}(a)
\end{aligned}
\end{equation*}
Together with Eq.~\eqref{eq:thm1-cumulative-value-gap}, this shows that $\mathcal{M}_{-}\ne\varnothing$ and that there exists $\ell\in \mathcal{M}_{-}$ such that
\begin{equation*}
\sum_{b=0}^{\infty}\pi_{b,0}(\ell)=+\infty.
\end{equation*}
Summing the inequality in Eq.~\eqref{eq:thm1-logit-direction-bounds} for action $\ell$ gives $\theta_{b,0}(\ell)\to-\infty$. Since $\theta_{b,0}(1)$ is bounded below,
\begin{equation*}
\frac{\pi_{b,0}(1)}{\pi_{b,0}(\ell)}
=\exp\bigl(\theta_{b,0}(1)-\theta_{b,0}(\ell)\bigr)
\longrightarrow+\infty.
\end{equation*}
Thus, for all sufficiently large $b$, we have $\pi_{b,0}(1)\ge\pi_{b,0}(\ell)$, and hence
\begin{equation*}
\sum_{b=0}^{\infty}\pi_{b,0}(1)=+\infty.
\end{equation*}
Applying Eq.~\eqref{eq:thm1-logit-direction-bounds} again gives $\theta_{b,0}(1)\to+\infty$. For every $a\in \mathcal{M}_{-}$, $\theta_{b,0}(a)$ is nonincreasing from $B_0$ onward and is therefore bounded above. Hence
\nopagebreak[4]
\begin{equation*}
\frac{\pi_{b,0}(a)}{\pi_{b,0}(1)}
=\exp\bigl(\theta_{b,0}(a)-\theta_{b,0}(1)\bigr)
\longrightarrow0,\qquad a\in \mathcal{M}_{-}.
\end{equation*}
Consequently,
\begin{equation*}
\begin{aligned}
J(\pi_{b,0})-J_\infty
&=\sum_{a\in \mathcal{M}_{+}}(\mu(a)-J_\infty)\pi_{b,0}(a)
-\sum_{a\in \mathcal{M}_{-}}(J_\infty-\mu(a))\pi_{b,0}(a)\\
&\ge\pi_{b,0}(1)\left[
\mu(1)-J_\infty
-\sum_{a\in \mathcal{M}_{-}}(J_\infty-\mu(a))
\frac{\pi_{b,0}(a)}{\pi_{b,0}(1)}\right]>0
\end{aligned}
\end{equation*}
This holds for all sufficiently large $b$, contradicting $J(\pi_{b,0})\le J_\infty$. Hence $J_\infty=\mu(1)$. Since action $1$ is uniquely optimal, $\mathcal{M}=\{1\}$, and Eq.~\eqref{eq:thm1-limiting-set-mass} gives
\begin{equation*}
\pi_{b,0}\longrightarrow e_1,\qquad J(\pi_{b,0})\longrightarrow\mu(1).
\end{equation*}

Since $J(\pi_{b,s}) \geq J(\pi_{b,0})$, we have
\begin{equation*}
    0 \leq \mumax - J(\pi_{b,s}) \leq \mumax - J(\pi_{b,0}) \rightarrow 0.
\end{equation*}
And we have the lower bound
\begin{equation*}
    \mumax - J(\pi_{b,s}) \geq \Delta(1-\pi_{b,s}(1)).
\end{equation*}
Therefore
\begin{equation*}
    \max_{0\leq s \leq S}\| \pi_{b,s} - e_1 \|_1 \leq \frac{2(\mumax - J(\pi_{b,0}))}{\Delta} \rightarrow 0.
\end{equation*}
This completes the proof.

\section{Proofs for Section~\ref{subsec:theory_convergence_rates}}
\label{app:conv-rate-lemma}

\subsection{Proof of Lemma~\ref{lem:inner-loop-progress}}
\label{app:inner loop-progress}

\paragraph{Constants.}
\label{app:rate-constants}
The constants in Lemma~\ref{lem:inner-loop-progress} are
\begin{equation*}
 A=4+6\eta \mumax,\qquad
 \rho=\frac{\mumin\Delta}{\mumax^2}\exp\!\left(-\frac{\mumax^2}{\mumin\Delta}\right).
\end{equation*}
\begin{equation*}
 \lambda\big(\pi_{b,0}(1)\big)=\left(1-\frac{\eta \mumax}{4}\right)
 \frac{(\pi_{b,0}(1))^2}{8\sqrt2}\log\!\left(1+\frac{\Delta}{2\mumax}\right).
\end{equation*}
Here $0<\rho<1$ and $0<\lambda(\pi_{b,0}(1))<1$.

Focusing on the $b$-th outer iteration, we write
\begin{align*}
q \coloneqq \pi_{b,0}, \quad \ir \coloneqq \rg(q)>0, \quad p_k \coloneqq \pi_{b,k}, \quad g_k \coloneqq g_{b,k}
\end{align*}
for notational simplicity. Finite logits and the unique optimum ensure $\ir>0$. We keep the inner loop index on logits only when needed.

For this fixed $q$, Eq.~\eqref{eq:surrogate-ascent} ensures that
$L(p_{k+1};q)\geq L(p_k;q)$ for every $0\leq k<S$. By Lemma \ref{lem:grad order},
\begin{equation}
 \|g_k\|_2\leq\|g_0\|_2\leq\|g_0\|_1
 =\sum_aq(a)|\ir-(\mumax-\mu(a))|
 \leq\sum_aq(a)(\ir+\mumax-\mu(a))=2\ir.
 \label{eq:app-gradient-gap}
\end{equation}

\paragraph{Lower bound.}
Let
\begin{equation*}
 r_k=\max_a(\lo_{b,k}(a)-\lo_{b,0}(a))-\min_a(\lo_{b,k}(a)-\lo_{b,0}(a)).
\end{equation*}
Note that
\begin{equation}
 \frac{p_k(a)}{q(a)}
 =\frac{\exp(\lo_{b,k}(a)-\lo_{b,0}(a))}
 {\sum_j q(j)\exp(\lo_{b,k}(j)-\lo_{b,0}(j))}.
 \label{eq:app-softmax-ratio}
\end{equation}

Eq.~\eqref{eq:app-softmax-ratio} gives $e^{-r_k}\leq p_k(a)/q(a)\leq e^{r_k}$. Then we have
\begin{equation}
 p_k(1)\leq\frac{q(1)e^{r_k}}{1-q(1)+q(1)e^{r_k}},
 \qquad
 p_k(1)-q(1)\leq q(1)(1-q(1))(e^{r_k}-1).
 \label{eq:app-optimal-probability-change}
\end{equation}
Also, $\ir=\sum_{a>1}(\mumax-\mu(a))q(a)\geq\Delta(1-q(1))$.
For any $1\leq s\leq S$, choose
\begin{equation*}
 N_s=\min\!\left\{s,
 \left\lfloor\frac{\log(1+\Delta/(2\mumax))}{2\sqrt2\eta\ir}\right\rfloor+1\right\}.
\end{equation*}
By Eq.~\eqref{eq:app-gradient-gap}, for every $k<N_s$,
\begin{equation*}
 r_k\leq\sqrt2\eta\sum_{j<k}\|g_j\|_2
 \leq2\sqrt2\eta k\ir
 \leq\log\!\left(1+\frac{\Delta}{2\mumax}\right).
\end{equation*}
Substituting into Eq.~\eqref{eq:app-optimal-probability-change} yields
\begin{equation*}
 p_k(1)-q(1)\leq\frac{q(1)(1-q(1))\Delta}{2\mumax}
 \leq\frac{q(1)\ir}{2\mumax}.
\end{equation*}
Since $0<J(\pi_{b,0})\leq \mumax$, this implies
\begin{equation*}
 g_k(1)=q(1)\ir-J(\pi_{b,0})(p_k(1)-q(1))
 \geq\frac12q(1)\ir
\end{equation*}
Sum Eq.~\eqref{eq:surrogate-ascent} over $k=0,\ldots,N_s-1$ and use
monotonicity between steps $N_s$ and $s$ to obtain
\begin{align*}
 L(p_s;q)-L(q;q)
 &\geq L(p_{N_s};q)-L(q;q)\nonumber\\
 &\geq\left(1-\frac{\eta \mumax}{4}\right)
 \frac{q(1)^2\eta N_s\ir^2}{4}\nonumber\\
 &\geq\left(1-\frac{\eta \mumax}{4}\right)
 \frac{q(1)^2\log(1+\Delta/(2\mumax))}{8\sqrt2}
 \min\{\eta s\ir^2,\ir\}.
\end{align*}
For the last inequality, set
$\alpha=\log(1+\Delta/(2\mumax))/(2\sqrt2)\in(0,1)$.
Then $N_s\geq\min\{s,\alpha/(\eta\ir)\}$, so
$\eta N_s\ir^2\geq\min\{\eta s\ir^2,\alpha\ir\}
\geq \alpha\min\{\eta s\ir^2,\ir\}$.

\paragraph{Upper bound.}
We have
\begin{equation*}
 \nabla_\lo J(\pi_\lo)(a)=\pi_\lo(a)(\mu(a)-J(\pi_\lo)),\qquad
 \|\nabla_\lo J(\pi_\lo)\|_2\leq\|\nabla_\lo J(\pi_\lo)\|_1\leq2\rg(\pi_\lo).
\end{equation*}
Equations~\eqref{eq:surrogate-ascent} and~\eqref{eq:thm1-surrogate-value-bound} imply $\rg(p_k)\leq\ir$, hence $\|\nabla_\lo J(p_k)\|_2\leq2\ir$.
Differentiating once more gives
\begin{equation*}
 \nabla_\lo^2J(\pi_\lo)=\diag(\nabla_\lo J(\pi_\lo))-\nabla_\lo J(\pi_\lo)\pi_\lo^\top-\pi_\lo \nabla_\lo J(\pi_\lo)^\top.
\end{equation*}
Because $|\mu(a)-J(\pi_\lo)|\leq \mumax$, we have $\|\nabla_\lo J(\pi_\lo)\|_2\leq\|\nabla_\lo J(\pi_\lo)\|_1\leq \mumax$, $\|\diag(\nabla_\lo J(\pi_\lo))\|_2\leq \mumax$, and $\|\pi_\lo\|_2\leq1$. Therefore $\|\nabla_\lo^2 J(\pi_\lo)\|_2\leq3\mumax$.
From Taylor's theorem and Eq.~\eqref{eq:app-gradient-gap} we have
\begin{align*}
 J(p_{k+1})-J(p_k)
 &\leq\eta \nabla_\lo J(p_k)^\top g_k+\frac{3\mumax\eta^2}{2}\|g_k\|_2^2\nonumber\\
 &\leq4\eta\ir^2+6\mumax\eta^2\ir^2
 =(4+6\eta \mumax)\eta\ir^2.
\end{align*}
Summing over $k<s$ and using $0\leq J(p_s)-J(\pi_{b,0})\leq\ir$ proves the upper bound in Eq.~\eqref{eq:inner loop-progress}. The middle inequality is Eq.~\eqref{eq:thm1-surrogate-value-bound}.

\paragraph{A lower bound on the suboptimality gap.}
We prove a bound for any positive policy $p$ satisfying $L(p;q)\geq L(q;q)$, and then apply it to $p_s$.
The gap and reward bounds imply
\begin{equation*}
 \Delta\sum_{a>1}q(a)\leq\ir\leq \mumax\sum_{a>1}q(a),
 \qquad H_q:=\sum_{a>1}q(a)\mu(a)\geq \mumin\sum_{a>1}q(a)\geq\frac{\mumin\ir}{\mumax}>0.
\end{equation*}
The optimal action contributes at most
\begin{equation*}
 \mumax q(1)\log\frac{p(1)}{q(1)}
 \leq \mumax(p(1)-q(1))\leq \mumax(1-q(1))\leq\frac{\mumax\ir}{\Delta}.
\end{equation*}
Jensen's inequality for the suboptimal terms gives
\begin{align*}
 \sum_{a>1}q(a)\mu(a)\log\frac{p(a)}{q(a)}
 &=H_q\sum_{a>1}\frac{q(a)\mu(a)}{H_q}\log\frac{p(a)}{q(a)}\nonumber\\
 &\leq H_q\log\left(\sum_{a>1}\frac{q(a)\mu(a)}{H_q}\frac{p(a)}{q(a)}\right)
 =H_q\log\frac{\sum_{a>1}\mu(a) p(a)}{H_q}.
\end{align*}
Since $\rg(p)\geq\Delta\sum_{a>1}p(a)$, we have
\begin{equation*}
 \frac{\sum_{a>1}\mu(a) p(a)}{H_q}
 \leq\frac{\mumax \rg(p)}{\Delta H_q}
 \leq\frac{\mumax^2}{\mumin\Delta}\frac{\rg(p)}{\ir}.
\end{equation*}
Consequently,
\begin{equation*}
 0\leq L(p;q)-L(q;q)
 \leq\frac{\mumax\ir}{\Delta}
 +H_q\log\!\left(\frac{\mumax^2}{\mumin\Delta}\frac{\rg(p)}{\ir}\right).
\end{equation*}
Rearranging and using $H_q\geq \mumin\ir/\mumax$ gives
\begin{equation*}
 \log\!\left(\frac{\mumax^2}{\mumin\Delta}\frac{\rg(p)}{\ir}\right)
 \geq-\frac{\mumax\ir}{\Delta H_q}\geq-\frac{\mumax^2}{\mumin\Delta},
 \qquad \rg(p)\geq\rho\ir.
\end{equation*}
Taking $p=p_s$ and combining with $\rg(p_s)\leq\ir$ proves Eq.~\eqref{eq:inner loop-residual} in the lemma.

\subsection{Proofs of Theorems~\ref{thm:rate-upper} and~\ref{thm:rate-lower}}
\label{app:rate-profile}

\begin{lemma}
\label{lem:app-reciprocal-bounds}
Under the assumptions of Lemma~\ref{lem:inner-loop-progress},
$p_{\min}:=\inf_{b\geq0}\pi_{b,0}(1)>0$.
Set $\lambda_*=\lambda(p_{\min})$.
For every $b\geq0$ and $0\leq s\leq S$,
\begin{align}
 \frac1{\rg(\pi_{b,s})}-\frac1{\ir_b}
 &\geq\lambda_*\min\left\{\eta s,\frac1{\ir_b}\right\},
 \label{eq:app-reciprocal-lower}\\
 \frac1{\rg(\pi_{b,s})}-\frac1{\ir_b}
 &\leq\min\left\{\frac{A}{\rho}\eta s,
 \frac{1-\rho}{\rho}\frac1{\ir_b}\right\}.
 \label{eq:app-reciprocal-upper}
\end{align}
\end{lemma}

\begin{proof}
Using Lemma~\ref{lem:inner-loop-progress} and $\rg(\pi_{b,s})\leq\ir_b$ give
\begin{align*}
 \frac1{\rg(\pi_{b,s})}-\frac1{\ir_b}
 &=\frac{J(\pi_{b,s})-J(\pi_{b,0})}
 {\rg(\pi_{b,s})\ir_b}
 \geq\lambda_*\min\left\{\eta s,\frac1{\ir_b}\right\}.
\end{align*}
For the upper bound, let
$\xi=(\ir_b-\rg(\pi_{b,s}))/\ir_b$.
Then
\begin{equation*}
 0\leq\xi\leq A\eta s\ir_b,\qquad \xi\leq1-\rho,\qquad
 \frac1{\rg(\pi_{b,s})}-\frac1{\ir_b}
 =\frac{\xi}{(1-\xi)\ir_b}.
\end{equation*}
Bounding the last expression with these two inequalities proves
Eq.~\eqref{eq:app-reciprocal-upper}.
At $s=0$, both sides are zero.
\end{proof}

For $x,y>0$, define the comparison sequence
\begin{equation}
 a^{x,y}_{0,0}=\ir_0^{-1},\qquad
 a^{x,y}_{b,s}=a^{x,y}_{b,0}
 +\min\{x\eta s,ya^{x,y}_{b,0}\},\qquad
 a^{x,y}_{b+1,0}=a^{x,y}_{b,S}.
 \label{eq:app-comparison-recursion}
\end{equation}
Set
\begin{equation}
 b_{x,y}=
 \begin{cases}
 0,&\eta S\ir_0\leq y/x,\\[2pt]
 \displaystyle\left\lceil
 \frac{\log(x\eta S\ir_0/y)}{\log(1+y)}
 \right\rceil,&\eta S\ir_0>y/x.
 \end{cases}
 \label{eq:app-transition-index}
\end{equation}
Define the following function:
\begin{equation}
 \mathcal C_{x,y}(t,S,\ir_0)=
 \begin{cases}
 \displaystyle\ir_0^{-1}+x\eta t,
 &\eta S\ir_0\leq y/x,\\[3pt]
 \displaystyle\frac{(1+y)^b}{\ir_0}
 +\min\left\{x\eta s,\frac{y(1+y)^b}{\ir_0}\right\},
 &\eta S\ir_0>y/x,\ b<b_{x,y},\\[3pt]
 \displaystyle\frac{(1+y)^{b_{x,y}}}{\ir_0}
 +x\eta(t-Sb_{x,y}),
 &\eta S\ir_0>y/x,\ b\geq b_{x,y}.
 \end{cases}
 \label{eq:app-profile-closed}
\end{equation}

Then we have:
\begin{lemma}
\label{lem:app-scalar-comparison}
The map $a\mapsto a+\min\{x\eta s,ya\}$ is increasing for every $s\geq0$. For $t=bS+s$ with $0\leq s<S$, we have
$a^{x,y}_{b,s} = \mathcal C_{x,y}(t,S,\ir_0)$.
\end{lemma}

\begin{proof}
The map is increasing because $\min\{x\eta s,ya\}$ is nondecreasing in $a$.
To compute the sequence, first consider the complete-stage recurrence
\begin{equation*}
a^{x,y}_{b+1,0}=
\begin{cases}
(1+y)a^{x,y}_{b,0},&a^{x,y}_{b,0}<x\eta S/y,\\
a^{x,y}_{b,0}+x\eta S,&a^{x,y}_{b,0}\geq x\eta S/y.
\end{cases}
\end{equation*}
Since the sequence is increasing, once it reaches $x\eta S/y$,
every subsequent stage adds the same amount $x\eta S$.

If $\eta S\ir_0\leq y/x$, then
$a^{x,y}_{0,0}=\ir_0^{-1}\geq x\eta S/y$.
The additive update therefore applies from the first stage, giving
$a^{x,y}_{b,0}=\ir_0^{-1}+bx\eta S$.

If $\eta S\ir_0>y/x$, each complete stage initially multiplies the
current value by $1+y$.
The first stage index at which the threshold is reached is the smallest
integer $b$ satisfying
\begin{equation*}
\frac{(1+y)^b}{\ir_0}\geq\frac{x\eta S}{y},
\end{equation*}
which is exactly $b_{x,y}$ in Eq.~\eqref{eq:app-transition-index}.
Consequently,
\begin{equation*}
a^{x,y}_{b,0}=
\begin{cases}
\displaystyle\frac{(1+y)^b}{\ir_0},&0\leq b<b_{x,y},\\[4pt]
\displaystyle\frac{(1+y)^{b_{x,y}}}{\ir_0}
+x\eta S(b-b_{x,y}),&b\geq b_{x,y}.
\end{cases}
\end{equation*}

Finally, the definition of $a^{x,y}_{b,s}$ adds
$\min\{x\eta s,ya^{x,y}_{b,0}\}$ to the stage-start value.
Before the threshold is reached, this minimum remains as written.
Once $a^{x,y}_{b,0}\geq x\eta S/y$, it equals $x\eta s$ because $s<S$.
Substituting the stage-start values above and using $t=bS+s$
gives the three branches of Eq.~\eqref{eq:app-profile-closed}.
\end{proof}

\begin{proof}[Proof of main theorems.]
By Lemmas~\ref{lem:app-reciprocal-bounds}
and~\ref{lem:app-scalar-comparison}, induction over complete stages,
starting from $1/\ir_0$, gives
\begin{equation*}
 a^{\lambda_*,\lambda_*}_{b,s}
 \leq\frac1{\rg(\pi_{b,s})}
 \leq a^{A/\rho,(1-\rho)/\rho}_{b,s}.
\end{equation*}
At the final incomplete stage, apply
Eqs.~\eqref{eq:app-reciprocal-lower}--\eqref{eq:app-reciprocal-upper}
with its actual number $s$ of completed steps.
Then we have the two-sided bound
\begin{equation*}
 \frac1{\mathcal C_{A/\rho,(1-\rho)/\rho}(t,S,\ir_0)}
 \leq\rg(\pi_{b,s})
 \leq\frac1{\mathcal C_{\lambda_*,\lambda_*}(t,S,\ir_0)}.
\end{equation*}

\paragraph{Upper bound.}
Set $c=\lambda_*$ and let $\btran=b_{c,c}$.
For $0\leq b\leq\btran$, Eq.~\eqref{eq:app-profile-closed} gives
\begin{equation*}
 \rg(\pi_{b,s})\leq\ir_0(1+c)^{-b},\qquad 0\leq s\leq S.
\end{equation*}
At $b=\btran$, it also gives $\ir_{\btran}\leq1/(\eta S)$.
By monotonicity, we have
$1/\ir_b\geq\eta S$ for every $b\geq\btran$.
For these stages, the minimum in Eq.~\eqref{eq:app-reciprocal-lower}
equals $\eta s$.
Summing over complete stages and then the final $s$ steps yields
\begin{equation*}
 \rg(\pi_{b,s})\leq
 \frac1{\ir_{\btran}^{-1}+c\eta(t-\btran S)},
 \qquad b\geq\btran,\quad t=bS+s.
\end{equation*}

\paragraph{Lower bound.}
For the upper comparison sequence, define
\begin{equation}
 \bar{b}^{\prime}
 = b_{A/\rho,(1-\rho)/\rho} =
 \begin{cases}
 0,&\eta S\ir_0\leq(1-\rho)/A,\\[2pt]
 \displaystyle\left\lceil
 \frac{\log(A\eta S\ir_0/(1-\rho))}{\log(1/\rho)}
 \right\rceil,&\eta S\ir_0>(1-\rho)/A.
 \end{cases}
 \label{eq:app-lower-transition}
\end{equation}
Here $1+y=1/\rho$. Eq.~\eqref{eq:app-profile-closed} gives the bound
\begin{equation*}
 \rg(\pi_{b,s})\geq
 \left[
 \frac{\rho^{-b}}{\ir_0}
 +\min\left\{\frac{A}{\rho}\eta s,
 \frac{(1-\rho)\rho^{-b}}{\rho\ir_0}\right\}
 \right]^{-1},\qquad b<\bar{b}^{\prime}.
\end{equation*}
By Eq.~\eqref{eq:inner loop-residual},
$\ir_{b+1}\geq\rho\ir_b$.
Induction therefore gives, for every $b\geq0$,
\begin{equation}
 \rg(\pi_{b,0})\geq\ir_0\rho^b,\qquad
 \rg(\pi_{b,S})\geq\ir_0\rho^{b+1}.
 \label{eq:app-lower-endpoints}
\end{equation}

For $b\geq \bar{b}^{\prime}$, Eq.~\eqref{eq:app-profile-closed} gives
\begin{equation*}
 \rg(\pi_{b,s})\geq
 \frac1{\ir_0^{-1}\rho^{-\bar{b}^{\prime}}+(A/\rho)\eta(t-\bar{b}^{\prime}S)}.
\end{equation*}
\end{proof}

\subsection{Explicit constants for the convergence rates}
\label{app:ordered-proof}

\begin{lemma}
\label{lem:app-ordering}
If $0<\eta\mumax\leq2$ and
$\pi_{0,0}(1)\geq\pi_{0,0}(a)$ for every $a>1$, then
$\pi_{b,s}(1)\geq\pi_{b,s}(a)$ for every $b\geq0$ and $0\leq s\leq S$.
Consequently, $p_{\min}\geq1/K$.
\end{lemma}
\begin{proof}
Assume inductively that $q=\pi_{b,0}$ satisfies
$q(1)\geq q(a)$ for all $a>1$. Fix $a>1$ and write $\xi_s=\lo_{b,s}(1)-\lo_{b,s}(a)$.
Then $\xi_0=\log(q(1)/q(a))\geq0$, and
$q(1)\mumax-q(a)\mu(a)>0$ because action $1$ is uniquely optimal and $q(a)>0$.
Whenever $\xi_s\geq0$,
\begin{equation*}
 0\leq\pi_{b,s}(1)-\pi_{b,s}(a)
 =\frac{e^{\lo_{b,s}(a)}(e^{\xi_s}-1)}
 {\sum_{j \in \Acal} e^{\lo_{b,s}(j)}}
 \leq\frac{e^{\xi_s}-1}{e^{\xi_s}+1}
 =\tanh(\xi_s/2)\leq\frac{\xi_s}{2}.
\end{equation*}
The update therefore gives
\begin{align*}
 \xi_{s+1}
 &=\xi_s+\eta\big[q(1)\mumax-q(a)\mu(a)
 -J(q)(\pi_{b,s}(1)-\pi_{b,s}(a))\big]\\
 &\geq\left(1-\frac{\eta J(q)}{2}\right)\xi_s
 +\eta(q(1)\mumax-q(a)\mu(a))\geq0.
\end{align*}
The last inequality uses $\eta J(q)\leq\eta\mumax\leq2$.
Induction over $s$ preserves the ordering within the stage.
Taking $s=S$ transfers it to the next stage, so induction over $b$
proves the ordering at every iterate. In particular,
$\pi_{b,0}(1)\geq1/K$ for every $b$.
\end{proof}

\paragraph{Explicit rate constants.}
By Lemma~\ref{lem:app-ordering}, $p_{\min}\geq1/K$. Set $C_2=\lambda(1/K)$ and $C_1=A/\rho$.
Since $A\geq4$ and $0<\rho<1$, $(1-\rho)/\rho\leq C_1$.
Equations~\eqref{eq:app-reciprocal-lower} and~\eqref{eq:app-reciprocal-upper}
now imply
\begin{equation*}
 C_2\min\left\{\eta s,\frac1{\ir_b}\right\}
 \leq\frac1{\rg(\pi_{b,s})}-\frac1{\ir_b}
 \leq C_1\min\left\{\eta s,\frac1{\ir_b}\right\}.
\end{equation*}
Applying the same comparison as in Appendix~\ref{app:rate-profile} gives
\begin{equation*}
 \frac1{\mathcal C_{C_1,C_1}(t,S,\ir_0)}
 \leq\rg(\pi_{b,s})
 \leq\frac1{\mathcal C_{C_2,C_2}(t,S,\ir_0)},
 \qquad t=bS+s.
\end{equation*}
For completeness, writing $b_\beta=b_{\beta,\beta}$, the symmetric comparison
function is
\begin{equation*}
 \mathcal C_{\beta,\beta}(t,S,\ir_0)=
 \begin{cases}
 \displaystyle\ir_0^{-1}+\beta\eta t,
 &\eta S\ir_0\leq1,\\[3pt]
 \displaystyle\frac{(1+\beta)^b}{\ir_0}
 +\beta\min\left\{\eta s,\frac{(1+\beta)^b}{\ir_0}\right\},
 &\eta S\ir_0>1,\ b<b_\beta,\\[3pt]
 \displaystyle\frac{(1+\beta)^{b_\beta}}{\ir_0}
 +\beta\eta(t-Sb_\beta),
 &\eta S\ir_0>1,\ b\geq b_\beta.
 \end{cases}
\end{equation*}
If $\eta S\ir_0>1$, then
$b_\beta=\lceil\log(\eta S\ir_0)/\log(1+\beta)\rceil$, otherwise $b_\beta=0$.

%% file: appendix/comparison_proofs.tex
\section{Proofs for Section~\ref{subsec:theory_benefits_of_off_policyness}}
\label{app:comparison}

Throughout this section, we consider the problem settings as described in Section~\ref{subsec:theory_benefits_of_off_policyness},
with an action space of size $K \ge 3$.
For a full-support policy $q$ and positive reward means $\mu$, write
\begin{equation*}
\widehat\pi_q(a):=\frac{q(a)\mu(a)}{J(q)},\qquad a \in \Acal.
\end{equation*}
When $q=\pi_{b,0}$, this is the reward-weighted rollout distribution $\pitildemub$ from Section~\ref{sec:main_results_high_level}. 
All constants denoted by $c,C>0$ below, whose values can change across different parts of our analysis, are independent of the initial optimal-action probability $x$ and parameters like $B, S, T$ that account for the total number of iterations in the \res algorithm.

\subsection{Part 1 of Theorem~\ref{thm:benefits_off_policyness}: lower bound for on-policy RE(1)}
\label{app:comparison-onpolicy}

For $S=1$, each stage consists of one gradient step, so $\pi_{b,0}$ is reached after exactly $b$ steps. We will show that reaching the target accuracy requires passing through a policy that assigns probability $1-O(x^{K-1})$ to action $2$. The following estimate then bounds how quickly the remaining
probability can grow.

\begin{lemma}
\label{lem:comparison-probability-growth}
Fix an action $j$ and let $m_b(j)=1-\pi_{b,0}(j)$.
Then we have
\begin{equation}
\frac1{m_{b+1}(j)}\geq\frac1{m_b(j)}-2\eta\mumax.
\label{eq:comparison-reciprocal-growth}
\end{equation}
Therefore, for integers $b,n\geq0$ such that
$2\eta\mumax n m_b(j)<1$,
\begin{equation}
\pi_{b+n,0}(a)\leq m_{b+n}(j)
\leq\frac{m_b(j)}{1-2\eta\mumax n m_b(j)},\qquad  \forall a\ne j.
\label{eq:comparison-probability-envelope}
\end{equation}
\end{lemma}

\begin{proof}
For $a\ne j$, we have
$|g_{b,0}(a)|\leq\mumax\pi_{b,0}(a)\leq\mumax m_b(j)$.
Since $\sum_a g_{b,0}(a)=0$, also
$|g_{b,0}(j)|\leq\sum_{a\ne j}|g_{b,0}(a)|\leq\mumax m_b(j)$.
The update of each probability ratio gives
\begin{align*}
\frac{m_{b+1}(j)}{1-m_{b+1}(j)}
&=\sum_{a\ne j}\frac{\pi_{b,0}(a)}{\pi_{b,0}(j)}
\exp\!\bigl(\eta(g_{b,0}(a)-g_{b,0}(j))\bigr)\\
&\leq\frac{m_b(j)}{1-m_b(j)}\exp(2\eta\mumax m_b(j)).
\end{align*}
Policy probability remain positive at finite iterates for each action. Using $e^{-v}\geq1-v$, we obtain
\begin{align*}
\frac1{m_{b+1}(j)}-1
&\geq\left(\frac1{m_b(j)}-1\right)e^{-2\eta\mumax m_b(j)}\\
&\geq\left(\frac1{m_b(j)}-1\right)(1-2\eta\mumax m_b(j)).
\end{align*}
Hence
$1/m_{b+1}(j)\geq1/m_b(j)-2\eta\mumax(1-m_b(j))
\geq1/m_b(j)-2\eta\mumax$.
Summing this inequality over $n$ updates and taking reciprocals when
the resulting lower bound is positive proves
Eq.~\eqref{eq:comparison-probability-envelope}.
\end{proof}

\begin{proof}[Proof of the first part of Theorem~\ref{thm:benefits_off_policyness}]
For $S=1$, Eqs.~\eqref{eq:surrogate-ascent}
and~\eqref{eq:thm1-surrogate-value-bound} imply $J(\pi_{b+1,0})\geq J(\pi_{b,0})$.
We first use this monotonicity to obtain a uniformly negative advantage
for every action $a\geq3$.
Set
\begin{equation*}
\overline\mu:=\max_{3\leq a\leq K}\mu(a),\qquad
\kappa:=\frac{\mu(2)-\overline\mu}{4}>0.
\end{equation*}
Consider the separate bandit on actions $2,\ldots,K$, initialized at
$w$. It has a unique optimal action $2$, finite initial logits
$\log w(a)$, and $\eta\mu(2)<4$. Theorem~\ref{thm:global}, applied
with $S=1$, shows that its mean reward converges to $\mu(2)$.
Choose a fixed integer $b_0\geq1$ at which this mean reward is at least
$\overline\mu+2\kappa$.

The full on-policy update has the probability representation
\begin{equation}
\pi_{b+1,0}(a)
=\frac{\pi_{b,0}(a)\exp(\eta g_{b,0}(a))}
{\sum_{j=1}^K\pi_{b,0}(j)\exp(\eta g_{b,0}(j))}.
\label{eq:comparison-probability-update}
\end{equation}

When $x=0$, action~1 has zero probability and the update reduces
to that of the bandit on actions $2,\ldots,K$, initialized at $w$.
By the choice of $b_0$, this reduced bandit has mean reward at least
$\overline\mu+2\kappa$ after $b_0$ updates.
The update map is continuous in the policy, including at the
boundary of the simplex. Thus, for sufficiently small $x>0$,
the mean reward after the same $b_0$ updates differs from that of
the reduced bandit by less than $\kappa$, giving
$J(\pi_{b_0,0})\geq\overline\mu+\kappa$.
Since the mean reward is nondecreasing and
$\mu(a)\leq\overline\mu$ for every $a\geq3$, it follows that
\begin{equation}
J(\pi_{b,0})-\mu(a)\geq\kappa,
\qquad b\geq b_0,\quad 3\leq a\leq K.
\label{eq:comparison-inferior-advantages}
\end{equation}
Here $b_0$ is chosen from the reduced bandit and depends not on $x$.

\paragraph{Probability concentration on the suboptimal action.}
Define the cumulative increase of the optimal logit by
\begin{equation}
\ell_b=\eta\sum_{k=0}^{b-1}\pi_{k,0}(1)\rg(\pi_{k,0})
=\lo_{b,0}(1)-\lo_{0,0}(1).
\label{eq:comparison-optimal-logit-increase}
\end{equation}
This sequence is nondecreasing, and each one-step increment is in
$[0,\eta\mumax]$. The probability-ratio update and
$\sum_a g_{b,0}(a)=0$ give the exact identity
\begin{equation}
\prod_{a=2}^K\frac{\pi_{b,0}(a)}{\pi_{b,0}(1)}
=\frac{(1-x)^{K-1}\prod_{a=2}^K w(a)}{x^{K-1}}
e^{-K\ell_b}.
\label{eq:comparison-probability-product}
\end{equation}
Indeed, the one-step change in the logarithm of the left-hand side is
$\eta\sum_{a=2}^K(g_{b,0}(a)-g_{b,0}(1))=-K\eta g_{b,0}(1)$. 

Now we define $b_\epsilon=\Teps(1)$. Since
\begin{equation*}
\rg(\pi_{b_\epsilon,0})
\geq\Delta\bigl(1-\pi_{b_\epsilon,0}(1)\bigr),
\end{equation*}
the condition $\epsilon<\Delta/2$ implies
$\pi_{b_\epsilon,0}(1)>1/2$. Every factor on the left-hand side of
Eq.~\eqref{eq:comparison-probability-product} is then less than one,
so
\begin{equation*}
\ell_{b_\epsilon}
>\frac1K\log\frac{(1-x)^{K-1}\prod_{a=2}^K w(a)}{x^{K-1}}
=\frac{K-1}{K}\log(1/x)+O(1).
\end{equation*}
On the other hand, $\ell_{b_0}\leq b_0\eta\mumax$.
For sufficiently small $x$, it follows that $b_\epsilon>b_0$ and
$\ell_{b_\epsilon}-\ell_{b_0}>1$. Thus the first crossing index
\begin{equation*}
b_{\mathrm c}:=\min\{b\geq b_0:\ell_b-\ell_{b_0}\geq1\}
\end{equation*}
exists and satisfies $b_0<b_{\mathrm c}\leq b_\epsilon$.
The one-step bound on $\ell_b$ gives
\begin{equation}
1\leq\ell_{b_{\mathrm c}}-\ell_{b_0}\leq1+\eta\mumax,
\qquad \ell_{b_{\mathrm c}}\leq1+(b_0+1)\eta\mumax.
\label{eq:comparison-crossing-increase}
\end{equation}

For $3\leq a\leq K$ and $b\geq b_0$, Eq.~\eqref{eq:comparison-inferior-advantages}
and the nonnegative optimal advantage imply
$g_{b,0}(1)-g_{b,0}(a)\geq\kappa\pi_{b,0}(a)$.
Using $e^v-1\geq v$ in the exact ratio update, we obtain
\begin{align}
\frac{\pi_{b+1,0}(1)}{\pi_{b+1,0}(a)}-\frac{\pi_{b,0}(1)}{\pi_{b,0}(a)}
&=\frac{\pi_{b,0}(1)}{\pi_{b,0}(a)}\left[e^{\eta(g_{b,0}(1)-g_{b,0}(a))}-1\right]\nonumber\\
&\geq\eta\kappa\pi_{b,0}(1)
\geq\frac{\kappa}{\mumax}(\ell_{b+1}-\ell_b).
\label{eq:comparison-inferior-ratio-growth}
\end{align}
Here the last inequality follows from
$\ell_{b+1}-\ell_b=\eta\pi_{b,0}(1)\rg(\pi_{b,0})
\leq\eta\mumax\pi_{b,0}(1)$.
Summing from $b_0$ to $b_{\mathrm c}-1$ gives
\begin{equation}
\frac{\pi_{b_{\mathrm c},0}(a)}{\pi_{b_{\mathrm c},0}(1)}
\leq\frac{\mumax}{\kappa},\qquad 3\leq a\leq K.
\label{eq:comparison-relative-compression}
\end{equation}

By Eq.~\eqref{eq:comparison-crossing-increase},
$\ell_{b_{\mathrm c}}$ has an upper bound independent of $x$.
Hence, for sufficiently small $x$,
Eq.~\eqref{eq:comparison-probability-product} gives
\[
\prod_{a=2}^K
\frac{\pi_{b_{\mathrm c},0}(a)}
{\pi_{b_{\mathrm c},0}(1)}
\ge c_1 x^{-(K-1)},
\]
where $c_1>0$ is independent of $x$.
By Eq.~\eqref{eq:comparison-relative-compression},
the product of the factors for actions $3,\ldots,K$
is at most $(\mumax/\kappa)^{K-2}$.
The remaining factor
therefore satisfies
\[
\frac{\pi_{b_{\mathrm c},0}(2)}
{\pi_{b_{\mathrm c},0}(1)}
\ge
\frac{c_1 x^{-(K-1)}}{(\mumax/\kappa)^{K-2}}
=
c_2 x^{-(K-1)},
\]
where $c_2:=c_1(\kappa/\mumax)^{K-2}>0$
is also independent of $x$. Thus, we further have
\[
\pi_{b_{\mathrm c},0}(1)\le c_2^{-1}x^{K-1}.
\]
Defining $C_1:=c_2^{-1}\max\{1,\mumax/\kappa\}$, using Eq.~\eqref{eq:comparison-relative-compression} and summing over $a\ne2$, we obtain
\begin{equation}
\begin{aligned}
\pi_{b_{\mathrm c},0}(a)
&\le C_1x^{K-1}\quad(a\ne2),\\
m_{b_{\mathrm c}}(2)=\sum_{a\ne2}\pi_{b_{\mathrm c},0}(a)
&\le (K-1)C_1 x^{K-1},
\end{aligned}
\label{eq:comparison-compressed-policy}
\end{equation}

In particular, $\pi_{b_{\mathrm c},0}(1)<1/2$ for sufficiently small $x$. On the other hand, by the definition of $b_\epsilon$ and $\epsilon<\Delta/2$, we have $\pi_{b_\epsilon,0}(1)>1/2$. Since $b_{\mathrm c}\le b_\epsilon$, it follows that $b_{\mathrm c}<b_\epsilon$.

\paragraph{The steps required to reach the target.}
Note that
$m_{b_\epsilon}(2)\geq\pi_{b_\epsilon,0}(1)>1/2$.
Using Lemma~\ref{lem:comparison-probability-growth}, summing Eq.~\eqref{eq:comparison-reciprocal-growth} from $b_{\mathrm c}$ to
$b_\epsilon-1$ and using
Eq.~\eqref{eq:comparison-compressed-policy}, we obtain
\begin{equation*}
\Teps(1)=b_\epsilon\geq b_\epsilon-b_{\mathrm c}
\geq\frac{(m_{b_{\mathrm c}}(2))^{-1}-2}{2\eta\mumax}
\geq \frac{1}{4\eta\mumax (K-1)C_1}x^{-(K-1)}
\end{equation*}
for all sufficiently small $x$. This proves the first part of
Theorem~\ref{thm:benefits_off_policyness}.
\end{proof}

\subsection{Part 2 of Theorem~\ref{thm:benefits_off_policyness}: upper bound for off-policy RE(S)}
\label{app:comparison-offpolicy}

We first bound the KL divergence to a fixed reward-weighted rollout distribution. We then use this bound to show that a sufficiently long inner loop preserves part of the target's improvement in the optimal action log-odds.

\begin{lemma}
\label{lem:comparison-fitting}
For a stage beginning at $q=\pi_{b,0}$ and any integer $S\geq1$,
\begin{equation}
\DKL(\widehat\pi_q\Vert\pi_{b,S})
\leq
\frac{\|\log\mu\|_2^2}
{2\eta\mumin(1-\eta\mumax/4)S},
\label{eq:comparison-fitting-bound}
\end{equation}
where $\|\log\mu\|_2^2=\sum_{a=1}^K(\log\mu(a))^2$.
\end{lemma}

\begin{proof}
The frozen target is represented by the logits
\begin{equation}
\lo^*:=\lo_{b,0}+\log\mu,
\qquad
\|\lo^*-\lo_{b,0}\|_2^2=\|\log\mu\|_2^2.
\label{eq:comparison-target-logits}
\end{equation}
Indeed, $\pi_{\lo^*}=\widehat\pi_q$, so $\lo^*$ maximizes
$L(\pi_\lo;q)$ and its gradient vanishes there.
Let $e_s=\lo^*-\lo_{b,s}$ and denote the surrogate gap by
\[
\delta_s:=L(\widehat\pi_q;q)-L(\pi_{b,s};q)
=J(\pi_{b,0})\DKL(\widehat\pi_q\Vert\pi_{b,s}).
\]
The surrogate is concave with a $J(\pi_{b,0})/2$-Lipschitz
gradient by Lemma~\ref{lem:lsmooth} and
Eq.~\eqref{eq:thm1-surrogate-hessian}. Co-coercivity therefore gives
$\|g_{b,s}\|_2^2\leq
(J(\pi_{b,0})/2)\langle e_s,g_{b,s}\rangle$.
Using $e_{s+1}=e_s-\eta g_{b,s}$ and concavity, we have
\begin{align*}
\|e_s\|_2^2-\|e_{s+1}\|_2^2
&=2\eta\langle e_s,g_{b,s}\rangle-\eta^2\|g_{b,s}\|_2^2\\
&\geq2\eta\left(1-\frac{\eta J(\pi_{b,0})}{4}\right)
\langle e_s,g_{b,s}\rangle\\
&\geq2\eta\left(1-\frac{\eta J(\pi_{b,0})}{4}\right)\delta_s.
\end{align*}
By Eq.~\eqref{eq:surrogate-ascent}, $\delta_s$ is nonincreasing.
Summing over the inner loop and using $\|e_S\|_2^2\geq0$ gives
\[
S\delta_S
\leq\sum_{s=0}^{S-1}\delta_s
\leq
\frac{\|\log\mu\|_2^2}
{2\eta(1-\eta J(\pi_{b,0})/4)}.
\]
Dividing by $J(\pi_{b,0})S$ and using
$\mumin\leq J(\pi_{b,0})\leq\mumax$ proves the claim.
\end{proof}

\begin{lemma}
\label{lem:comparison-binary}
Let $\widehat\pi,p$ be positive probability vectors,
$\alpha=\widehat\pi(1)$, and $0<v\leq1/2$.
Writing $\logit(u):=\log(u/(1-u))$, we have
\begin{equation}
\DKL(\widehat\pi\Vert p)
\leq\frac18\min\{\alpha,1-\alpha\}v^2
\quad\Longrightarrow\quad
\logit(p(1))\geq\logit(\alpha)-v.
\label{eq:comparison-binary-implication}
\end{equation}
\end{lemma}

\begin{proof}
If $p(1)\geq\alpha$, the conclusion is immediate. Otherwise, set $d=\logit(\alpha)-\logit(p(1))>0$.
Combining all actions other than action~1 by the log-sum inequality,
we obtain
\[
\DKL(\widehat\pi\Vert p)
\geq\alpha\log\frac{\alpha}{p(1)}
 +(1-\alpha)\log\frac{1-\alpha}{1-p(1)}
=f(d),
\]
where $f(t):=\alpha t+\log(1-\alpha+\alpha e^{-t})$.
Here $f(0)=f'(0)=0$, and $f$ is strictly increasing for $t>0$.
For $0\leq t\leq v\leq1/2$,
\[
f''(t)=
\frac{\alpha(1-\alpha)e^{-t}}{(1-\alpha+\alpha e^{-t})^2}
\geq\frac{\alpha(1-\alpha)}2
\geq\frac14\min\{\alpha,1-\alpha\}.
\]
Integrating twice gives
$f(v)\geq\min\{\alpha,1-\alpha\}v^2/8$.
If $d>v$, then $f(d)>f(v)$, contradicting the assumed KL bound. Therefore we have $d\leq v$.
\end{proof}

\begin{proof}[Proof of the second part of Theorem~\ref{thm:benefits_off_policyness}]
Set $h:=1-\epsilon/(\mumax-\mumin)>1/2$.
When $\pi_{b,0}(1)\geq h$, we have
$\rg(\pi_{b,0})\leq\epsilon$. In the following, assume that $\pb$ is sufficiently small so that
\begin{equation}
0<\pb\leq\frac{\mumin(1-h)}{\mumax}
=\frac{\mumin\epsilon}{\mumax(\mumax-\mumin)}.
\label{eq:comparison-small-initialization}
\end{equation}
This is a fixed positive upper bound and implies $\pb<h$.

\paragraph{Progress in one stage.}
Consider a stage beginning at $q=\pi_{b,0}$ with $\pb\leq q(1)<h$,
and let $\alpha=\widehat\pi_q(1)$.
Both parts of the target distribution have mass at least $\pb$:
\[
\alpha=\frac{\mumax q(1)}{J(\pi_{b,0})}\geq q(1)\geq\pb,
\qquad
1-\alpha\geq
\frac{\mumin(1-q(1))}{\mumax}
>\frac{\mumin(1-h)}{\mumax}\geq\pb.
\]
The target reward-weighted rollout distribution increases the optimal action's log-odds by
\begin{align}
\logit(\alpha)-\logit(q(1))
&=\log\frac{\mumax(1-q(1))}{\sum_{a>1}q(a)\mu(a)}
\nonumber\\
&\geq\log\frac{\mumax}{\mu(2)}
\geq\frac{\Delta}{\mumax}.
\label{eq:comparison-exact-target-progress}
\end{align}
To preserve at least half of this lower bound after the inner loop updates,
choose a fixed coefficient
\begin{equation}
c_S=
\frac{16\mumax^2\|\log\mu\|_2^2}
{\eta\mumin(1-\eta\mumax/4)\Delta^2},
\qquad
S_\pb:=\left\lceil\frac{c_S}{\pb}\right\rceil.
\label{eq:comparison-sufficient-budgets}
\end{equation}
Use the same number of gradient steps, $S_\pb$, in every inner loop.
For $v:=\Delta/(2\mumax)\in(0,1/2)$,
Lemma~\ref{lem:comparison-fitting} gives
\[
\DKL (\widehat\pi_q\Vert\pi_{b,S_\pb})
\leq\frac{\pb\Delta^2}{32\mumax^2}
=\frac{\pb v^2}{8}
\leq\frac18\min\{\alpha,1-\alpha\}v^2.
\]
Lemma~\ref{lem:comparison-binary} and
Eq.~\eqref{eq:comparison-exact-target-progress} now yield
\begin{equation}
\logit(\pi_{b+1,0}(1))
\geq\logit(\alpha)-\frac{\Delta}{2\mumax}
\geq\logit(\pi_{b,0}(1))+\frac{\Delta}{2\mumax},
\label{eq:comparison-endpoint-progress}
\end{equation}
where $\pi_{b+1,0}=\pi_{b,S_\pb}$.

\paragraph{The number of stages.}
Let $b_*:=\inf\{b\geq0:\pi_{b,0}(1)\geq h\}$, with $b_*=\infty$ if the set is empty.
Until this threshold is reached, Eq.~\eqref{eq:comparison-endpoint-progress}
shows that the optimal action probability increases after each inner loop and stays at least $\pb$. Induction therefore gives
\begin{equation}\label{eq:optimal-progress}
\logit(\pi_{b,0}(1))
\geq\logit(\pb)+\frac{b\Delta}{2\mumax}
\end{equation}
for every integer $0\leq b\leq b_*$. If $b_*=\infty$, the bound holds for all $b\geq0$. Since Eq.~\eqref{eq:comparison-small-initialization}
implies $\pb\leq1-h$,
\[
\logit(h)-\logit(\pb)
=\log\frac{h(1-\pb)}{(1-h)\pb}
\leq2\log(1/\pb).
\]
Define
\begin{equation*}
B_\pb:=\left\lceil\frac{4\mumax}{\Delta}\log(1/\pb)\right\rceil.
\end{equation*}
We claim that $b_*\leq B_\pb$.
Suppose for contradiction that $b_*>B_\pb$. Then using Eq.~\eqref{eq:optimal-progress} at $b=B_\pb$ we have
\[
\begin{aligned}
\logit(\pi_{B_\pb,0}(1))
&\geq\logit(\pb)+\frac{B_\pb\Delta}{2\mumax}\\
&\geq\logit(\pb)+2\log(1/\pb)
\geq\logit(h).
\end{aligned}
\]
Since $\logit$ is strictly increasing, this implies $\pi_{B_\pb,0}(1)\geq h$, contradicting the definition of $b_*$. Hence we have $b_*\leq B_\pb$. At $b=b_*$,
\[
\rg(\pi_{b_*,0})
\leq(\mumax-\mumin)(1-\pi_{b_*,0}(1))
\leq\epsilon.
\]
Thus
\begin{equation}\label{eq:RES-upperbound}
\Teps(S_\pb)\leq b_*S_\pb\leq B_\pb S_\pb
=O\!\left(\pb^{-1}\log(1/\pb)\right).
\end{equation}
Here $c_S>0$ is fixed independently of $\pb$, so
$S_\pb=\Theta(\pb^{-1})$.
Since $K\geq3$, combining Eq.~\eqref{eq:RES-upperbound} with the first part of the theorem gives
\[
\frac{\Teps(S_\pb)}{\Teps(1)}
\leq C\pb^{K-2}\log(1/\pb)\longrightarrow0
\qquad\text{as }\pb\downarrow0.
\]
\end{proof}

%% file: main.bbl
\begin{thebibliography}{41}
\providecommand{\natexlab}[1]{#1}
\providecommand{\url}[1]{\texttt{#1}}
\expandafter\ifx\csname urlstyle\endcsname\relax
  \providecommand{\doi}[1]{doi: #1}\else
  \providecommand{\doi}{doi: \begingroup \urlstyle{rm}\Url}\fi

\bibitem[Agarwal et~al.(2021)Agarwal, Kakade, Lee, and Mahajan]{agarwal2021theory}
Alekh Agarwal, Sham~M. Kakade, Jason~D. Lee, and Gaurav Mahajan.
\newblock {On the Theory of Policy Gradient Methods: Optimality, Approximation, and Distribution Shift}.
\newblock \emph{Journal of Machine Learning Research}, 22\penalty0 (98):\penalty0 1--76, 2021.

\bibitem[Arnal et~al.(2026)Arnal, Narozniak, Cabannes, Tang, Kempe, and Munos]{arnal2025asymmetric}
Charles Arnal, Ga{\"e}tan Narozniak, Vivien Cabannes, Yunhao Tang, Julia Kempe, and Remi Munos.
\newblock Asymmetric reinforce for off-policy reinforcement learning: Balancing positive and negative rewards.
\newblock \emph{Advances in Neural Information Processing Systems}, 38:\penalty0 9640--9664, 2026.

\bibitem[Asad et~al.(2025)Asad, Harikandeh, Laradji, Roux, and Vaswani]{asad2025spma}
Reza Asad, Reza~Babanezhad Harikandeh, Issam~H. Laradji, Nicolas~Le Roux, and Sharan Vaswani.
\newblock Fast convergence of softmax policy mirror ascent.
\newblock In \emph{Proceedings of The 28th International Conference on Artificial Intelligence and Statistics}, volume 258 of \emph{Proceedings of Machine Learning Research}, pp.\  3943--3951. PMLR, 03--05 May 2025.

\bibitem[Dai et~al.(2026)Dai, Liu, Zhang, and Wen]{dai2026nonasymptoticglobalconvergenceppoclip}
Qiming Dai, Yin Liu, Junyu Zhang, and Zaiwen Wen.
\newblock Non-asymptotic global convergence of ppo-clip.
\newblock \emph{arXiv}, 2026.

\bibitem[DeepSeek-AI(2025)]{deepseek-r1}
DeepSeek-AI.
\newblock {DeepSeek-R1}: Incentivizing reasoning capability in {LLMs} via reinforcement learning.
\newblock \emph{arXiv}, 2025.

\bibitem[Dong et~al.(2023)Dong, Xiong, Goyal, Zhang, Chow, Pan, Diao, Zhang, Shum, and Zhang]{dong2023raft}
Hanze Dong, Wei Xiong, Deepanshu Goyal, Yihan Zhang, Winnie Chow, Rui Pan, Shizhe Diao, Jipeng Zhang, Kashun Shum, and Tong Zhang.
\newblock {RAFT: Reward rAnked FineTuning for Generative Foundation Model Alignment}.
\newblock \emph{Transactions on Machine Learning Research}, 2023.

\bibitem[Espeholt et~al.(2018)Espeholt, Soyer, Munos, Simonyan, Mnih, Ward, Doron, Firoiu, Harley, Dunning, Legg, and Kavukcuoglu]{espeholt2018impala}
Lasse Espeholt, Hubert Soyer, Remi Munos, Karen Simonyan, Vlad Mnih, Tom Ward, Yotam Doron, Vlad Firoiu, Tim Harley, Iain Dunning, Shane Legg, and Koray Kavukcuoglu.
\newblock {IMPALA}: Scalable distributed deep-{RL} with importance weighted actor-learner architectures.
\newblock In \emph{Proceedings of the 35th International Conference on Machine Learning}, volume~80 of \emph{Proceedings of Machine Learning Research}, pp.\  1407--1416. PMLR, 10--15 Jul 2018.

\bibitem[Ghosh et~al.(2020)Ghosh, Machado, and Le~Roux]{ghosh2020operator}
Dibya Ghosh, Marlos~C. Machado, and Nicolas Le~Roux.
\newblock {An Operator View of Policy Gradient Methods}.
\newblock In \emph{Advances in Neural Information Processing Systems}, volume~33, 2020.

\bibitem[Gulcehre et~al.(2023)Gulcehre, Paine, Srinivasan, Konyushkova, Weerts, Sharma, Siddhant, Ahern, Wang, Gu, Macherey, Doucet, Firat, and de~Freitas]{gulcehre2023reinforcedselftrainingrestlanguage}
Caglar Gulcehre, Tom~Le Paine, Srivatsan Srinivasan, Ksenia Konyushkova, Lotte Weerts, Abhishek Sharma, Aditya Siddhant, Alex Ahern, Miaosen Wang, Chenjie Gu, Wolfgang Macherey, Arnaud Doucet, Orhan Firat, and Nando de~Freitas.
\newblock Reinforced self-training (rest) for language modeling.
\newblock \emph{arXiv}, 2023.

\bibitem[Hu et~al.(2024)Hu, Wu, Zhu, Xianyu, Wang, Zhang, and Cao]{openrlhf}
Jian Hu, Xibin Wu, Zilin Zhu, Xianyu, Weixun Wang, Dehao Zhang, and Yu~Cao.
\newblock {OpenRLHF}: An easy-to-use, scalable and high-performance {RLHF} framework.
\newblock \emph{arXiv preprint arXiv:2405.11143}, 2024.

\bibitem[Kimi-Team(2025)]{kimiteam2025kimik15scalingreinforcement}
Kimi-Team.
\newblock Kimi k1.5: Scaling reinforcement learning with {LLMs}.
\newblock \emph{arXiv preprint arXiv:2501.12599}, 2025.

\bibitem[Kwon et~al.(2023)Kwon, Li, Zhuang, Sheng, Zheng, Yu, Gonzalez, Zhang, and Stoica]{kwon2023efficientmemorymanagementlarge}
Woosuk Kwon, Zhuohan Li, Siyuan Zhuang, Ying Sheng, Lianmin Zheng, Cody~Hao Yu, Joseph~E. Gonzalez, Hao Zhang, and Ion Stoica.
\newblock Efficient memory management for large language model serving with pagedattention.
\newblock \emph{arXiv}, 2023.

\bibitem[Laroche \& Tachet~des Combes(2021)Laroche and Tachet~des Combes]{laroche2021jekyll}
Romain Laroche and Remi Tachet~des Combes.
\newblock Dr jekyll \& mr hyde: the strange case of off-policy policy updates.
\newblock In \emph{Advances in Neural Information Processing Systems}, volume~34, pp.\  24442--24454, 2021.

\bibitem[Lattimore \& Szepesv{\'a}ri(2020)Lattimore and Szepesv{\'a}ri]{lattimore2020bandit}
Tor Lattimore and Csaba Szepesv{\'a}ri.
\newblock \emph{Bandit Algorithms}.
\newblock Cambridge University Press, 2020.

\bibitem[Li et~al.(2023)Li, Wei, Chi, and Chen]{li2023exponential}
Gen Li, Yuting Wei, Yuejie Chi, and Yuxin Chen.
\newblock Softmax policy gradient methods can take exponential time to converge.
\newblock \emph{Mathematical Programming}, 201\penalty0 (1):\penalty0 707--802, 2023.

\bibitem[Liu et~al.(2024)Liu, Li, and Wei]{liu2024elementaryanalysispolicygradient}
Jiacai Liu, Wenye Li, and Ke~Wei.
\newblock Elementary analysis of policy gradient methods.
\newblock \emph{arXiv}, 2024.

\bibitem[Lu et~al.(2024)Lu, Aghaei, Raj, and Vaswani]{lu2024towards}
Michael Lu, Matin Aghaei, Anant Raj, and Sharan Vaswani.
\newblock Towards principled, practical policy gradient for bandits and tabular mdps.
\newblock \emph{arXiv preprint arXiv:2405.13136}, 2024.

\bibitem[Mei \& Osband(2026)Mei and Osband]{mei2026delightful}
Jincheng Mei and Ian Osband.
\newblock {Delightful Gradients Accelerate Corner Escape}.
\newblock \emph{arXiv preprint arXiv:2605.11908}, 2026.

\bibitem[Mei et~al.(2020{\natexlab{a}})Mei, Xiao, Dai, Li, Szepesv{\'a}ri, and Schuurmans]{mei2020gravity}
Jincheng Mei, Chenjun Xiao, Bo~Dai, Lihong Li, Csaba Szepesv{\'a}ri, and Dale Schuurmans.
\newblock {Escaping the Gravitational Pull of Softmax}.
\newblock In \emph{Advances in Neural Information Processing Systems}, volume~33, 2020{\natexlab{a}}.

\bibitem[Mei et~al.(2020{\natexlab{b}})Mei, Xiao, Szepesv{\'a}ri, and Schuurmans]{mei2020global}
Jincheng Mei, Chenjun Xiao, Csaba Szepesv{\'a}ri, and Dale Schuurmans.
\newblock {On the Global Convergence Rates of Softmax Policy Gradient Methods}.
\newblock In \emph{Proceedings of the 37th International Conference on Machine Learning}, volume 119 of \emph{Proceedings of Machine Learning Research}, pp.\  6820--6829, 2020{\natexlab{b}}.

\bibitem[Mei et~al.(2023)Mei, Zhong, Dai, Agarwal, Szepesvari, and Schuurmans]{mei2023stochastic}
Jincheng Mei, Zixin Zhong, Bo~Dai, Alekh Agarwal, Csaba Szepesvari, and Dale Schuurmans.
\newblock Stochastic gradient succeeds for bandits.
\newblock In \emph{International Conference on Machine Learning}, pp.\  24325--24360. PMLR, 2023.

\bibitem[Nesterov(2013)]{nesterov2013introductory}
Yurii Nesterov.
\newblock \emph{Introductory lectures on convex optimization: A basic course}.
\newblock Springer Science \& Business Media, 2013.

\bibitem[{OpenAI}(2024)]{openai-o1}
{OpenAI}.
\newblock {OpenAI} o1 system card.
\newblock \emph{arXiv Preprint arXiv:2412.16720}, 2024.

\bibitem[Ouyang et~al.(2022)Ouyang, Wu, Jiang, Almeida, Wainwright, Mishkin, Zhang, Agarwal, Slama, Ray, Schulman, Hilton, Kelton, Miller, Simens, Askell, Welinder, Christiano, Leike, and Lowe]{ouyang2022training}
Long Ouyang, Jeffrey Wu, Xu~Jiang, Diogo Almeida, Carroll Wainwright, Pamela Mishkin, Chong Zhang, Sandhini Agarwal, Katarina Slama, Alex Ray, John Schulman, Jacob Hilton, Fraser Kelton, Luke Miller, Maddie Simens, Amanda Askell, Peter Welinder, Paul~F Christiano, Jan Leike, and Ryan Lowe.
\newblock Training language models to follow instructions with human feedback.
\newblock In \emph{Advances in Neural Information Processing Systems}, volume~35, pp.\  27730--27744. Curran Associates, Inc., 2022.

\bibitem[Pan et~al.(2025)Pan, Chen, Chen, Sun, Chen, Zhang, Xie, Huang, Zhang, Gao, Shi, Li, Ding, and Zhou]{trinity}
Xuchen Pan, Yanxi Chen, Yushuo Chen, Yuchang Sun, Daoyuan Chen, Wenhao Zhang, Yuexiang Xie, Yilun Huang, Yilei Zhang, Dawei Gao, Weijie Shi, Yaliang Li, Bolin Ding, and Jingren Zhou.
\newblock Trinity-{RFT}: A general-purpose and unified framework for reinforcement fine-tuning of large language models.
\newblock \emph{arXiv Preprint arXiv:2505.17826}, 2025.

\bibitem[Qin \& Springenberg(2025)Qin and Springenberg]{qin2025curated}
Chongli Qin and Jost~Tobias Springenberg.
\newblock {Supervised Fine Tuning on Curated Data is Reinforcement Learning (and can be improved)}.
\newblock \emph{arXiv preprint arXiv:2507.12856}, 2025.

\bibitem[Robertson et~al.(2025)Robertson, Chu, Dai, Schuurmans, Szepesvari, and Mei]{robertson2025reinforce}
Samuel Robertson, Thang Chu, Bo~Dai, Dale Schuurmans, Csaba Szepesvari, and Jincheng Mei.
\newblock Reinforce converges to optimal policies with any learning rate.
\newblock In \emph{Advances in Neural Information Processing Systems}, 2025.

\bibitem[Russo(2026)]{russo2026success}
Daniel Russo.
\newblock {Success Conditioning as Policy Improvement: The Optimization Problem Solved by Imitating Success}.
\newblock \emph{arXiv preprint arXiv:2601.18175}, 2026.

\bibitem[Schulman et~al.(2015)Schulman, Levine, Abbeel, Jordan, and Moritz]{schulman2015trust}
John Schulman, Sergey Levine, Pieter Abbeel, Michael Jordan, and Philipp Moritz.
\newblock Trust region policy optimization.
\newblock In \emph{International conference on machine learning}, pp.\  1889--1897. PMLR, 2015.

\bibitem[Schulman et~al.(2017)Schulman, Wolski, Dhariwal, Radford, and Klimov]{schulman2017proximal}
John Schulman, Filip Wolski, Prafulla Dhariwal, Alec Radford, and Oleg Klimov.
\newblock Proximal policy optimization algorithms.
\newblock \emph{arXiv preprint arXiv:1707.06347}, 2017.

\bibitem[Sheng et~al.(2024)Sheng, Zhang, Ye, Wu, Zhang, Zhang, Peng, Lin, and Wu]{verl}
Guangming Sheng, Chi Zhang, Zilingfeng Ye, Xibin Wu, Wang Zhang, Ru~Zhang, Yanghua Peng, Haibin Lin, and Chuan Wu.
\newblock {HybridFlow}: A flexible and efficient {RLHF} framework.
\newblock \emph{arXiv}, 2024.

\bibitem[Singh et~al.(2024)Singh, Co-Reyes, Agarwal, Anand, Patil, Garcia, Liu, Harrison, Lee, Xu, Parisi, Kumar, Alemi, Rizkowsky, Nova, Adlam, Bohnet, Elsayed, Sedghi, Mordatch, Simpson, Gur, Snoek, Pennington, Hron, Kenealy, Swersky, Mahajan, Culp, Xiao, Bileschi, Constant, Novak, Liu, Warkentin, Qian, Bansal, Dyer, Neyshabur, Sohl-Dickstein, and Fiedel]{singh2024humandatascalingselftraining}
Avi Singh, John~D. Co-Reyes, Rishabh Agarwal, Ankesh Anand, Piyush Patil, Xavier Garcia, Peter~J. Liu, James Harrison, Jaehoon Lee, Kelvin Xu, Aaron Parisi, Abhishek Kumar, Alex Alemi, Alex Rizkowsky, Azade Nova, Ben Adlam, Bernd Bohnet, Gamaleldin Elsayed, Hanie Sedghi, Igor Mordatch, Isabelle Simpson, Izzeddin Gur, Jasper Snoek, Jeffrey Pennington, Jiri Hron, Kathleen Kenealy, Kevin Swersky, Kshiteej Mahajan, Laura Culp, Lechao Xiao, Maxwell~L. Bileschi, Noah Constant, Roman Novak, Rosanne Liu, Tris Warkentin, Yundi Qian, Yamini Bansal, Ethan Dyer, Behnam Neyshabur, Jascha Sohl-Dickstein, and Noah Fiedel.
\newblock Beyond human data: Scaling self-training for problem-solving with language models.
\newblock \emph{arXiv}, 2024.

\bibitem[{\v{S}}trupl et~al.(2022){\v{S}}trupl, Faccio, Ashley, Srivastava, and Schmidhuber]{strupl2022rwr}
Miroslav {\v{S}}trupl, Francesco Faccio, Dylan~R Ashley, Rupesh~Kumar Srivastava, and J{\"u}rgen Schmidhuber.
\newblock {Reward-Weighted Regression Converges to a Global Optimum}.
\newblock \emph{Proceedings of the AAAI Conference on Artificial Intelligence}, 36\penalty0 (8):\penalty0 8361--8369, 2022.

\bibitem[Sutton \& Barto(1998)Sutton and Barto]{sutton1998reinforcement}
Richard~S Sutton and Andrew~G Barto.
\newblock \emph{Reinforcement learning: An introduction}.
\newblock MIT press Cambridge, 1998.

\bibitem[Sutton et~al.(1999)Sutton, McAllester, Singh, and Mansour]{sutton1999policy}
Richard~S Sutton, David McAllester, Satinder Singh, and Yishay Mansour.
\newblock Policy gradient methods for reinforcement learning with function approximation.
\newblock \emph{Advances in neural information processing systems}, 12, 1999.

\bibitem[Touvron et~al.(2023)Touvron, Martin, Stone, Albert, Almahairi, Babaei, Bashlykov, Batra, Bhargava, Bhosale, Bikel, Blecher, Ferrer, Chen, Cucurull, Esiobu, Fernandes, Fu, Fu, Fuller, Gao, Goswami, Goyal, Hartshorn, Hosseini, Hou, Inan, Kardas, Kerkez, Khabsa, Kloumann, Korenev, Koura, Lachaux, Lavril, Lee, Liskovich, Lu, Mao, Martinet, Mihaylov, Mishra, Molybog, Nie, Poulton, Reizenstein, Rungta, Saladi, Schelten, Silva, Smith, Subramanian, Tan, Tang, Taylor, Williams, Kuan, Xu, Yan, Zarov, Zhang, Fan, Kambadur, Narang, Rodriguez, Stojnic, Edunov, and Scialom]{touvron2023llama2openfoundation}
Hugo Touvron, Louis Martin, Kevin Stone, Peter Albert, Amjad Almahairi, Yasmine Babaei, Nikolay Bashlykov, Soumya Batra, Prajjwal Bhargava, Shruti Bhosale, Dan Bikel, Lukas Blecher, Cristian~Canton Ferrer, Moya Chen, Guillem Cucurull, David Esiobu, Jude Fernandes, Jeremy Fu, Wenyin Fu, Brian Fuller, Cynthia Gao, Vedanuj Goswami, Naman Goyal, Anthony Hartshorn, Saghar Hosseini, Rui Hou, Hakan Inan, Marcin Kardas, Viktor Kerkez, Madian Khabsa, Isabel Kloumann, Artem Korenev, Punit~Singh Koura, Marie-Anne Lachaux, Thibaut Lavril, Jenya Lee, Diana Liskovich, Yinghai Lu, Yuning Mao, Xavier Martinet, Todor Mihaylov, Pushkar Mishra, Igor Molybog, Yixin Nie, Andrew Poulton, Jeremy Reizenstein, Rashi Rungta, Kalyan Saladi, Alan Schelten, Ruan Silva, Eric~Michael Smith, Ranjan Subramanian, Xiaoqing~Ellen Tan, Binh Tang, Ross Taylor, Adina Williams, Jian~Xiang Kuan, Puxin Xu, Zheng Yan, Iliyan Zarov, Yuchen Zhang, Angela Fan, Melanie Kambadur, Sharan Narang, Aurelien Rodriguez, Robert Stojnic, Sergey Edunov, and Thomas Scialom.
\newblock Llama 2: Open foundation and fine-tuned chat models.
\newblock \emph{arXiv}, 2023.

\bibitem[Vaswani et~al.(2022)Vaswani, Bachem, Totaro, M\"uller, Garg, Geist, Machado, Samuel~Castro, and Le~Roux]{vaswani2022surrogates}
Sharan Vaswani, Olivier Bachem, Simone Totaro, Robert M\"uller, Shivam Garg, Matthieu Geist, Marlos~C. Machado, Pablo Samuel~Castro, and Nicolas Le~Roux.
\newblock A general class of surrogate functions for stable and efficient reinforcement learning.
\newblock In \emph{Proceedings of The 25th International Conference on Artificial Intelligence and Statistics}, volume 151 of \emph{Proceedings of Machine Learning Research}, pp.\  8619--8649. PMLR, 28--30 Mar 2022.

\bibitem[Williams(1992)]{williams1992simple}
Ronald~J Williams.
\newblock Simple statistical gradient-following algorithms for connectionist reinforcement learning.
\newblock \emph{Machine learning}, 8\penalty0 (3):\penalty0 229--256, 1992.

\bibitem[Zelikman et~al.(2022)Zelikman, Wu, Mu, and Goodman]{zelikman2022star}
Eric Zelikman, Yuhuai Wu, Jesse Mu, and Noah Goodman.
\newblock Star: Bootstrapping reasoning with reasoning.
\newblock In \emph{Advances in Neural Information Processing Systems}, volume~35, pp.\  15476--15488. Curran Associates, Inc., 2022.
\newblock \doi{10.52202/068431-1126}.

\bibitem[Zhang et~al.(2022)Zhang, des Combes, and Laroche]{zhang2022offpolicy}
Shangtong Zhang, Remi~Tachet des Combes, and Romain Laroche.
\newblock Global optimality and finite sample analysis of softmax off-policy actor critic under state distribution mismatch.
\newblock \emph{Journal of Machine Learning Research}, 23\penalty0 (343):\penalty0 1--91, 2022.

\bibitem[Zheng et~al.(2024)Zheng, Yin, Xie, Sun, Huang, Yu, Cao, Kozyrakis, Stoica, Gonzalez, Barrett, and Sheng]{2024sglangefficientexecutionofstructuredlanguagemodelprograms}
Lianmin Zheng, Liangsheng Yin, Zhiqiang Xie, Chuyue Sun, Jeff Huang, Cody~Hao Yu, Shiyi Cao, Christos Kozyrakis, Ion Stoica, Joseph~E. Gonzalez, Clark Barrett, and Ying Sheng.
\newblock Sglang: Efficient execution of structured language model programs.
\newblock In \emph{Conference on Neural Information Processing Systems (NeurIPS)}, 2024.

\end{thebibliography}
